\documentclass[sigconf]{acmart}
\copyrightyear{2026}
\acmYear{2026}
\setcopyright{cc}
\setcctype{by}
\acmConference[KDD '26]{Proceedings of the 32nd ACM SIGKDD Conference on Knowledge Discovery and Data Mining V.2}{August 09--13, 2026}{Jeju Island, Republic of Korea}
\acmBooktitle{Proceedings of the 32nd ACM SIGKDD Conference on Knowledge Discovery and Data Mining V.2 (KDD '26), August 09--13, 2026, Jeju Island, Republic of Korea}
\acmDOI{10.1145/3770855.3817569}
\acmISBN{979-8-4007-2259-2/2026/08}

\usepackage{algorithm}
\usepackage{algorithmic}
\usepackage{hyperref}
\usepackage{multirow}

\newcommand{\vecK}{\mathbf{k}}
\newcommand{\vecX}{\mathbf{x}}
\newcommand{\vecS}{\mathbf{s}}
\newcommand{\vecY}{\mathbf{y}}
\newcommand{\vecMu}{\boldsymbol{\mu}}
\newcommand{\xDim}{p}
\newcommand{\dDim}{n}

\begin{document}

\title{Benchmarking on Tasks That Matter: Dataset Selection for Preserving Model Rankings}

% \affiliation{%

% \affiliation{%

% \affiliation{%

% \affiliation{%

% \affiliation{%

% \affiliation{%

\author{Rostislav Gusev}
\affiliation{%
  \institution{Applied AI Institute}
  \city{Moscow}
  \state{Moscow region}
  \country{Russian Federation}}
\email{gavkav05@gmail.com}
\orcid{0009-0006-6289-9551}
% \email{Rostislav.Gusev@skoltech.ru}

\author{Alexey Zaytsev}
\affiliation{%
  \institution{Applied AI Institute}
  \city{Moscow}
  \state{Moscow region}
  \country{Russian Federation}}
\email{likzet@gmail.com}

% \affiliation{%
%   \institution{Skolkovo Institute of Science and Technology}
%   \city{Moscow}
%   \country{Russia}}
% \email{A.Zaytsev@skoltech.ru}

% \renewcommand{\shortauthors}{R. Gusev \& A. Zaytsev}
\renewcommand{\shortauthors}{Rostislav Gusev and Alexey Zaytsev}

\begin{abstract}
Benchmarks of machine learning models often include many datasets, making evaluation expensive. 
For efficiency, it is preferable to perform evaluations on small, representative datasets instead. 
The selection of such subsets typically relies on heuristics and is rarely analyzed for the robustness of the resulting model rankings. 

We introduce a framework to perform the task of selecting datasets subsets with an evaluation of how different selection strategies preserve the global model rankings. 
Our framework includes bootstrap aggregation, which provides valid confidence intervals, allowing a principled comparison of selection strategies. 
We consider clustering, design criteria (A/D-optimality), random baselines, and greedy farthest-first (FAFI).
For the latter, we derive upper bounds on selection quality in terms of ranking errors as a function of the number of selected datasets.

Empirically, in time series classification (TSC, 112 datasets) and in a supplementary natural language processing benchmark derived from MTEB (57 tasks), several selection strategies improve rank preservation compared with random subsets, including simple FAFI. In contrast, in recommender systems (30 datasets), the improvement of strategies over random selection is small and typically statistically insignificant. 
For TSC, our best-performing strategy achieves a Spearman correlation of 0.95 with the full benchmark model rankings using only five selected datasets. 
Additional experiments indicate that the effectiveness of selection approaches depends on both the quality of dataset representations and the scale of the benchmarking regime.
\end{abstract}




\maketitle

\section{Introduction}

Multi-dataset benchmarks have become the default way to evaluate progress: rather than relying on a single task, new methods are expected to demonstrate robust improvements across collections of datasets. This paradigm is particularly prominent in time series classification (TSC), where community repositories such as UCR/UEA enable large-scale, standardized comparisons \cite{dau2019ucr,bagnall2016greattimeseriesclassification}. At the same time, exhaustive evaluation is increasingly resource-intensive: the number of datasets grows, and modern models (including deep architectures and ensembles) often make exhaustive model-by-model evaluation prohibitively expensive. Consequently, practitioners routinely evaluate on smaller subsets of datasets, sometimes chosen implicitly, manually, or simply by convenience.

This practice raises two challenges. First, conclusions drawn from a small subset can be statistically unstable: comparing methods on only a few datasets rarely supports reliable claims about overall superiority and requires careful multiple-dataset statistical procedures \cite{demsar2006statistical}. Second, subset choice can introduce selection bias: when the benchmark composition is flexible, the apparent leader may be an artifact of the selected tasks rather than a genuine improvement. This issue is closely related to the benchmark lottery effect, where rankings can change substantially under different task selections, especially when many models/configurations are compared \cite{dehghani2021benchmarklottery}.

\begin{figure*}[t]
  \centering
  \includegraphics[width=0.95\textwidth]{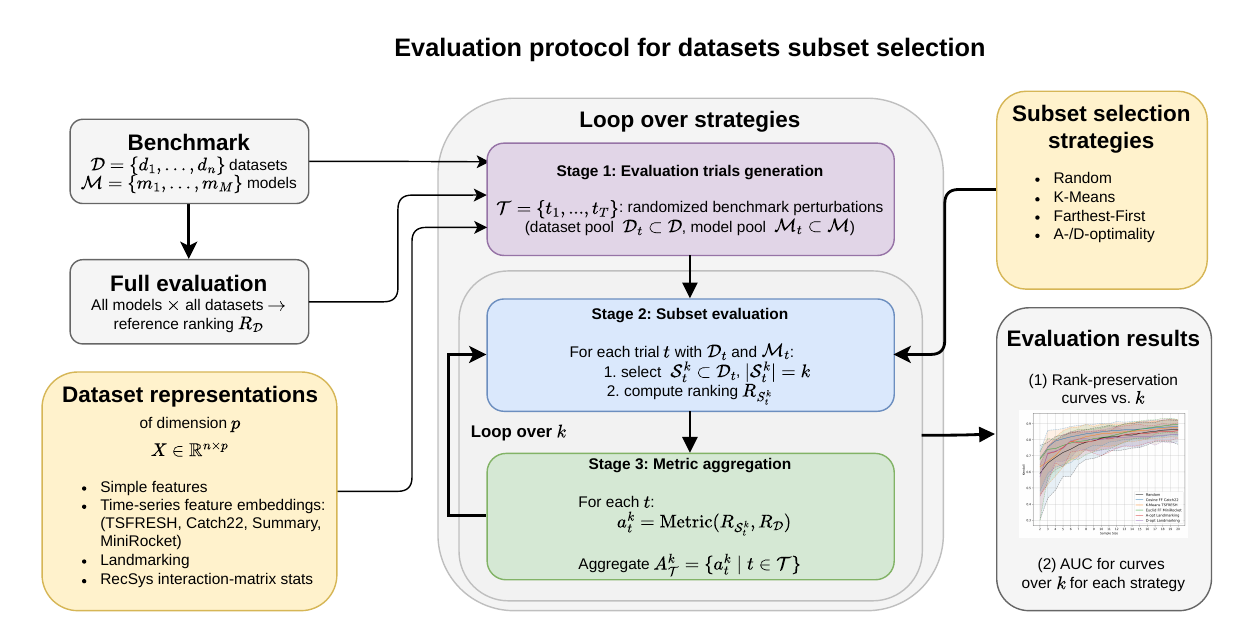}
  \caption{We study how to select a small subset of datasets that best preserves the global model ranking of a large benchmark. Datasets are embedded via task features, subsets are chosen by multiple strategies (geometric, clustering, design-based), and an evaluation protocol with bootstrap aggregation yields rank-preservation metrics, enabling comparisons across subset sizes and domains (time series classification, recommender systems, natural language processing).}
  \label{fig:teaser}
\end{figure*}

These concerns motivate the following question: \textbf{can we select a small but representative subset of datasets that substantially reduces evaluation cost while preserving the conclusions of full-benchmark model comparisons?} In this work, we treat datasets subset selection for benchmarking as a standalone problem. We focus on preserving the global ranking of models induced by the full benchmark and consider selection strategies that rely only on a priori dataset representations (meta-features, time-series feature extractors, landmarking features) rather than on full model evaluation results.

Our contributions are threefold. 
\begin{itemize}
    \item \textbf{A unified evaluation protocol for efficient dataset selection methods.} We propose a unified evaluation protocol for dataset selection methods, based on repeated subsampling of the available pool and bootstrap-style uncertainty estimation over performance-vs.-subset-size curves. 
    \item \textbf{Empirical evaluation for three domains: time series classification, RecSys, and NLP.} Within this protocol, we systematically compare several families of selection strategies --- random baselines, geometric diversity heuristics (e.g., farthest-first), clustering-based selection, and criteria inspired by optimal experimental design --- across multiple dataset representations in two domains (TSC and recommender systems), and additionally test the portability of the protocol on an NLP benchmark derived from MTEB~\cite{muennighoff2023mteb}.
    \item \textbf{Theoretical examination for the farthest-first selection method.} 
    For the farthest-first selection algorithm, we derive theoretical convergence guarantees linking geometric coverage in representation space to statistical accuracy: as the subset size grows, posterior uncertainty over model performance contracts and errors in model ranking decay at rates governed by the covering properties of farthest-first. We also conduct additional experiments and ablations that extend the empirical picture beyond aggregate benchmark results, helping to clarify when and why different selection principles succeed or fail across representations and domains.
\end{itemize}
The code is publicly available.\footnote{The code to reproduce all experiments, including the benchmarking protocol and dataset selection strategies, is available at   \url{https://github.com/hangover137/efficient_benchmarking}.}

\section{Related Work}
Benchmarking across multiple datasets is widely recognized as necessary for making reliable claims about algorithmic superiority \cite{demsar2006statistical,benavoli2015reallyuseposthoctests,bouthillier2021accountingvariancemachinelearning}. In TSC, the UCR/UEA ecosystem and associated bake-off studies have highlighted how evaluation outcomes depend on benchmark composition, protocols, and baseline choices \cite{dau2019ucr,bagnall2016greattimeseriesclassification,keogh2003needtsbenchmarks}. Beyond TSC, the \emph{benchmark lottery} perspective shows that even modest changes in task sets can produce different winners and unstable rankings when many models are compared \cite{dehghani2021benchmarklottery,recht2019imagenetclassifiersgeneralizeimagenet,yadav2019coldcaselostmnist,NEURIPS2019_ee39e503,bowman2021fixbenchmarkingnaturallanguage}. In recommender systems, large-scale benchmarking efforts likewise emphasize dataset heterogeneity and the sensitivity of algorithm rankings to the chosen dataset pool \cite{Shevchenko2024RecSysBenchmarking,Ferrari_Dacrema_2019,chin2022datasetsdilemma,bellogin2017statbias,castells2022offlineeval}. These lines of work collectively motivate studying rank-preserving reductions of benchmark size rather than optimizing for a fixed, small set.

Open benchmarking platforms aim to standardize protocols and improve reproducibility by publishing datasets, splits, and results, in line with the FAIR guiding principles \cite{wilkinson2016fair}. OpenML is a prominent example, enabling systematic large-scale comparisons through shared tasks and run records \cite{vanschoren2014openml}. Curated suites such as OpenML-CC18 further fix dataset selections and partitions to ensure comparability across studies \cite{bischl2021openmlbenchmarkingsuites,wang2020supergluestickierbenchmarkgeneralpurpose,zhai2020largescalestudyrepresentationlearning,thiyagalingam2021scientificmachinelearningbenchmarks}. However, curated suites are typically static: they do not address how to adaptively construct a compact benchmark when the available dataset pool, model zoo, or evaluation objectives change. Our setting complements these efforts by treating benchmark construction itself as an algorithmic selection problem with an explicit rank-preservation objective.

Selecting representative subsets is a classical problem in clustering and geometric approximation. The greedy farthest-first traversal (often viewed through the $k$-center / maximin lens) provides a simple, scalable way to obtain dispersed subsets and comes with approximation guarantees \cite{GONZALEZ1985293}. Similarly, $k$-means clustering and initialization schemes such as $k$-means++ are widely used to obtain representative prototypes in feature space \cite{arthur2007kmeanspp}. These methods are attractive as benchmark selection heuristics when datasets can be embedded into a feature space, but their effectiveness depends on whether feature-space proximity reflects similarity in model behavior—an assumption we examine empirically through our diagnostic analysis.

A different tradition formalizes selection using information-theoretic or variance-reduction criteria. In optimal experimental design, A- and D-optimality aim to choose "design points" that improve estimation accuracy and reduce uncertainty, often through properties of the Fisher information or covariance of estimators \cite{KieferWolfowitz1959OptDesign}. We adopt these principles as dataset-selection strategies, but evaluate them through a benchmarking-specific lens: the ability to preserve global model comparisons rather than to optimize coverage or parameter estimation per se.

Meta-learning and algorithm selection connect dataset characteristics to expected algorithm performance \cite{RICE197665}. This perspective has driven the development of meta-features and fast proxy evaluations (\texttt{landmarking}), as well as standardized repositories of cross-dataset results \cite{vanschoren2014openml,niessl2022overoptimism}. For time series, feature-based dataset characterizations are particularly mature: compact hand-engineered descriptors such as \texttt{Catch22} \cite{lubba2019catch22canonicaltimeseriescharacteristics}, automated extractors such as \texttt{TSFRESH} \cite{CHRIST201872}, and transform-based representations such as \texttt{MiniRocket} \cite{Dempster_2021} provide multiple ways to embed datasets into informative spaces. For recommender systems, analogous dataset-level representations can be derived from the user--item interaction matrix (e.g., size/shape, density and activity statistics, inequality measures such as Gini, and indicators of popularity bias and long-tail effects), enabling systematic comparisons and meta-analysis across heterogeneous datasets \cite{DELDJOO2021102662,chin2022datasetsdilemma,castells2022offlineeval}.

\section{Problem Definition}
\label{sec:problem_definition}

We consider the task of selecting a subset of datasets for comparative analysis, where the goal is to preserve the relative order of models in the domain under consideration, determined by testing on a complete set of datasets.

\paragraph{Benchmark Setting}
\label{sec:bench_sett}

Consider a set of $\dDim$ datasets:
\[
\mathcal{D} = \{d_1, d_2, \dots, d_{\dDim}\},
\]
together with a set of $M$ algorithms:
\[
\mathcal{M} = \{m_1, m_2, \dots, m_M\}.
\]
Each algorithm $ m \in \mathcal{M} $ is evaluated on each dataset $d \in \mathcal{D}$ using a fixed evaluation protocol, such as cross-validation with $F$ folds, to produce models. 
For each model, we have a scalar performance score for each model--dataset--fold triple.
We denote by $s^{(f)}_{m, d} $ the performance of a model produced by an algorithm $m$ on a dataset $d$ across evaluation folds $f \in \{1, \dots, F\}$.

In addition, we assume that for each dataset $ d \in \mathcal{D} $ there exists a vector representation
\begin{equation}
\label{eq:x_p}
\mathbf{x}_d \in \mathbb{R}^{\xDim},
\end{equation}
which characterizes the dataset independently of the evaluated models.
For example, such representations may be derived from dataset meta-features like sample size or dimensionality. 
Further details on specific representations are postponed to Section~\ref{sec:representations}.

\paragraph{True model ranking}

For each dataset $d$ and each evaluation fold $f$, models are ranked according to their fold-specific performance:
\[
r^{(f)}_d(m) =
\operatorname{rank}\bigl(
s^{(f)}_{m,d} : m \in \mathcal{M}
\bigr),
\]
where lower ranks correspond to better performance. We then define the average rank of a model $m$ on a dataset $d$ as the mean of its fold-wise ranks:
\[
\bar{r}_d(m) =
\frac{1}{F}
\sum_{f=1}^{F}
r^{(f)}_d(m).
\]
Finally, the global benchmark ranking induced by the full dataset collection $\mathcal{D}$ is defined as the average of these dataset-level ranks:
\[
R_{\mathcal{D}}(m) =
\frac{1}{\dDim}
\sum_{d \in \mathcal{D}}
\bar{r}_d(m).
\]

The vector
\[
\mathbf{R}_{\mathcal{D}} =
\{ R_{\mathcal{D}}(m) : m \in \mathcal{M} \}
\]
represents the reference ordering of models induced by the full benchmark.
It serves as the target object to be preserved under the dataset subset selection.
It is the main quantity of interest in the paper. Formally, we can go further by inducing a distribution on datasets $p(d)$ and saying that we consider an empirical substitute induced by $\mathcal{D}$. However, we leave this further generalization out of scope for the current paper.

\paragraph{Datasets subset selection}
We consider selecting a subset
\[
\mathcal{S} \subset \mathcal{D}, \qquad |\mathcal{S}| = k \ll |\mathcal{D}|.
\]
Evaluation on $\mathcal{S}$ induces an approximate global ranking
\begin{equation}
R_{\mathcal{S}}(m) = \frac{1}{|\mathcal{S}|} \sum_{d \in \mathcal{S}} \bar{r}_d(m),
\end{equation}
\begin{equation}
\mathbf{R}_{\mathcal{S}} =
\{ R_{\mathcal{S}}(m) : m \in \mathcal{M} \}
\end{equation}

The objective of the dataset subset selection for benchmarking is not to maximize performance on $\mathcal{S}$, but to ensure that $\mathbf{R}_{\mathcal{S}}$ is a faithful proxy of $\mathbf{R}_{\mathcal{D}}$:
\[
\mathbf{R}_{\mathcal{S}} \approx \mathbf{R}_{\mathcal{D}}.
\]

\paragraph{Rank preservation problem statement}
We formalize this problem as a ranking preservation problem.
Given a distance measure $\Delta$ between rankings, we want to select a subset of datasets $\mathcal{S} \subseteq \mathcal{D}$, $|\mathcal{S}| = k$, such that the discrepancy between the ranking obtained from the subset and the ranking obtained from the full sample $\Delta(\mathbf{R}_{\mathcal{S}}, \mathbf{R}_{\mathcal{D}}) $
is minimized:
\[
% \boxed{
\Delta(\mathbf{R}_{\mathcal{S}}, \mathbf{R}_{\mathcal{D}}) \rightarrow \min_{\mathcal{S}, |\mathcal{S}| = k}\!.
% }
\]

In this paper, we focus on rank-based metrics $\Delta(\mathbf{R}_{\mathcal{S}}, \mathbf{R}_{\mathcal{D}})$ that reflect differences in the relative ordering of models. 

To get an idea of how we define them within our notation, let us consider the mean absolute error (MAE), between rankings:
    \[
    \mathrm{MAE}(\mathbf{R}_\mathcal{S}, \mathbf{R}_{\mathcal{D}}) =
    \frac{1}{M} \sum_{j=1}^{M} \left| R_{\mathcal{S}} (j) - R_{\mathcal{D}} (j) \right|.
    \]
This metric quantifies the average positional shift of models when moving
from the full benchmark to the selected subset. 
Lower values indicate better preservation of the global ranking.

Similarly we use for the pair $\mathbf{R}_\mathcal{S}, \mathbf{R}_{\mathcal{D}}$
Spearman’s Rank Correlation $\rho(\mathbf{R}_\mathcal{S}, \mathbf{R}_{\mathcal{D}})$,
Kendall’s Rank Correlation $\tau(\mathbf{R}_\mathcal{S}, \mathbf{R}_{\mathcal{D}})$, Normalized Discounted Cumulative Gain (NDCG@K) with $K = 5$, and Mean Reciprocal Rank (MRR) used in our experiments.
For these four metrics, larger values indicate better preservation of the ranking.

\section{Dataset selection strategies}
\label{sec:selection_strategies}

In this section, we describe dataset subset selection strategies with subsequent theoretical motivation behind the farthest-first strategy.

Let each dataset $d \in \mathcal{D}$ with $\mathcal{D}$ being the full pool of datasets of size $|\mathcal{D}| = n$ be associated with a feature vector $\vecX \in \mathbb {R}^{\xDim}$, and the feature matrix of the entire pool be denoted by
$X \in \mathbb {R}^{n \times p}$, where the $i$-th row corresponds to a vector $\vecX_i^T$. The task is to select a subset $\mathcal{S} \subset \mathcal{D}$ of size $|\mathcal{S}| = k$, as described in Section~\ref{sec:evaluation_methodology}.

Below, we provide an overview of four strategies: random baseline, clustering, farthest-first algorithms with various distance metrics (FAFI, ours), and methods inspired by optimal experimental design (A-/D-optimality). 
In the implementation, feature matrices are standardized before selection for all methods.

\paragraph{Random sampling.}
We use uniform sampling of $k$ datasets without replacement as our baseline algorithm. 
This baseline allows us to interpret whether a geometric/structural strategy outperforms random sampling. 
Formally, from a dataset pool of size $t$, we select the subset $\mathcal{S}_k$ of size $k$ uniformly at random without replacement.

\paragraph{K-Means clustering}
This strategy selects datasets that represent typical regions of the feature space.
We cluster the datasets using k-means in the feature space, then select one representative per cluster, the one closest to the centroid. 
The formal procedure is summarized in Algorithm~\ref{alg:kmeans}.

\begin{algorithm}[t]
\caption{K-Means representative selection}
\label{alg:kmeans}
{\normalfont
\noindent
\begin{algorithmic}[1]
\REQUIRE Feature matrix $X \in \mathbb{R}^{n \times \xDim}$, subset size $k$
\ENSURE Subset $\mathcal{S} \subset \{1, \dots, n\}$, $|\mathcal{S}| = k$
\STATE (optional) standardize feature columns of $X$
\STATE Run $k$-means with $k$ clusters for $X$, obtaining centroids $\{\vecMu_c\}_{c=1}^k$
       and clusters with indexes of points that belong to them $\{\mathcal{C}_c\}_{c=1}^k$
\STATE $\mathcal{S} \gets \emptyset$
\FOR{$c = 1, \dots, k$}
  \STATE $i_c \gets \arg\min_{i \in \mathcal{C}_c} \lVert \vecX_i - \vecMu_c \rVert^2_2$
  \STATE $\mathcal{S} \gets \mathcal{S} \cup \{i_c\}$
\ENDFOR
\RETURN $\mathcal{S}$
\end{algorithmic}
}
\end{algorithm}

\paragraph{Greedy farthest-first selection.}   

To explicitly incentivize diversity and coverage, we use a greedy farthest-first \cite{GONZALEZ1985293} procedure: at each step, we add a dataset that maximizes the distance to the already selected set $\mathcal{S}$ in terms of the distance to the closest selected element. We consider two metrics: cosine and Euclidean. 
The resulting selection procedure is detailed in Algorithm~\ref{alg:farthest_first}.

\begin{algorithm}[t]
\caption{Greedy farthest-first dataset selection}
\label{alg:farthest_first}
{\normalfont
\noindent
\begin{algorithmic}[1]
\REQUIRE Feature matrix $X \in \mathbb{R}^{\dDim \times \xDim}$, subset size $k$, distance metric $d(\cdot,\cdot)$
\ENSURE Subset $\mathcal{S} \subset \{1, \dots, n\}$, $|\mathcal{S}| = k$
\STATE $\bar{\vecX} \gets \frac{1}{n} \sum_{i=1}^{n} \vecX_i$
\STATE $i_0 \gets \arg\max_{i \in \{1, \dots, n\}} d(\vecX_i, \bar{\vecX})$
\STATE $\mathcal{S} \gets \{ i_0 \}$
\WHILE{$|\mathcal{S}| < k$}
  \STATE $j^* \gets \arg\max_{j \notin \mathcal{S}} \min_{i \in \mathcal{S}} d(\vecX_j, \vecX_i)$
  \STATE $\mathcal{S} \gets \mathcal{S} \cup \{ j^{*} \}$
\ENDWHILE
\RETURN $\mathcal{S}$
\end{algorithmic}
}
\end{algorithm}

\paragraph{A-/D-optimality.}
We also consider strategies inspired by optimal experimental design.
The intuition is as follows: if we consider a linear regression model in a feature space, then the Fisher information matrix for a selected subset $\mathcal{S}_t$ is proportional to
\begin{equation}
    I(\mathcal{S}) = X_{\mathcal{S}}^\top X_{\mathcal{S}} + \lambda I,
\end{equation}
where $X_{\mathcal{S}} \in \mathbb{R}^{k \times \xDim}$ is the submatrix of rows corresponding to $\mathcal{S}$, and $\lambda > 0$ is a small regularization for numerical stability.
Then:
\begin{itemize}
\item \emph{A-optimality} minimizes the total variance of the estimates: $\Phi_A(\mathcal{S}) = \mathrm{tr} \big(I(\mathcal{S})^{-1}\big)$.
\item \emph{D-optimality} maximizes the amount of information: $\Phi_D(\mathcal{S}) = \log\det \big(I(\mathcal{S})\big)$.
\end{itemize}
Since exact optimization over all subsets is intractable for large $|\mathcal{D}|$, we use a greedy search that adds the dataset that improves a specific quality metric the most, augmented with a swap check, where we also greedily try to replace one dataset from $\mathcal{S}$ with some other from $\mathcal{D} \setminus \mathcal{S}$.
Versions with and without swaps provided low performance scores in experiments, 
while the swap variant was slightly better, so we report results for it. 

\subsection{Why farthest-first helps: a theoretical motivation.}
\label{sec:theory}

We provide a theoretical justification for cosine farthest-first (Alg. \ref{alg:farthest_first}) by linking coverage in the representation space to (i) the decay of the integrated uncertainty of performance estimates and (ii) the decay of ranking errors.  
Each dataset is represented by a feature vector $\vecX \in \mathbb{R}^{\xDim}$ and an associated positive semi-definite kernel $k(\cdot, \cdot)$ with $k(\vecX, \vecX) = 1$.
Let us consider a performance function $f: \mathcal{D} \to \mathbb{R}$, such that $f \sim \mathrm{GP}(0, k(\cdot, \cdot) + \sigma^2 \delta (\vecX - \vecX'))$, a realization of a Gaussian process with zero mean and covariance $k(\cdot, \cdot)$ with observation noise variance $\sigma^2$. 
We use the GP as a standard smooth surrogate prior over the representation space, as in Bayesian optimization for expensive ML performance functions~\cite{snoek2012practical} and GP-based NAS~\cite{zhang2020gpnas}.

We define the kernel-based dissimilarity
\begin{equation}
	d_k(\vecX, \vecX') = 2 \left(1 - k(\vecX, \vecX')\right).
\end{equation}
$\mathcal{D} = \{\vecX_{(i)}\}_{i = 1}^n$ constitutes an available finite pool of datasets.
Given the set of pairwise dissimilarities, we define the farthest-first procedure (FAFI) that 
produces a sequence $\vecX_1, \vecX_2, \dots$.
It starts from $\mathcal{S}_1 = \{\vecX\}$ such that $\vecX = \arg\max_{\vecX \in D} \frac{1}{n} \sum_{i} d_k(\vecX, \vecX_{i}).$ 
At each step, we add the furthest $\vecX_t$ among the remaining:
\begin{equation}
    \label{eq:fufi}
    \vecX_t  = \arg\max_{\vecX \in \mathcal{D}} d_k\bigl(\vecX, \mathcal{S}_{t - 1}\bigr), 
\end{equation}
with $\mathcal{S}_{t - 1} = \{\vecX_1, \dots, \vecX_{t - 1}\}$ and 
$d_k(\vecX, \mathcal{S}) = \min_{\vecX' \in \mathcal{S}} d_k(\vecX, \vecX')$.

Given $t$ selected points $\mathcal{S}_{t}$ and associated function values $\mathbf{f}_{t} = \{f(\vecX_i)\}_{i = 1}^{t}$, let $\mu_{t}(\vecX)$ and $\sigma^2_{t}(\vecX)$ be the GP posterior mean $\mathbb{E}( f(\vecX) | \mathbf{f}_{t})$ and variance $\mathbb{V}( f(\vecX) | \mathbf{f}_{t})$ at $\vecX$.
We are interested in the normalized mean squared error (MSE) of our predictions for the full sample $\mathcal{D}$:
\begin{equation}
E_t := \frac{1}{n} \sum_{\vecX \in \mathcal{D}}
\left(f(\vecX) - \mu_{t - 1}(\vecX)\right)^2.
\end{equation}
The estimation of $E_t$ reflects how our predictions $\mu_{t - 1}(\vecX)$ deviate from the true values $f(\vecX)$ on average. 
Combining a concentration inequality with geometric bounds on the covering radius obtained by furthest-first selection to obtain an upper bound for $E_t$, we obtain the quality for FAFI.

\begin{theorem}[MSE decay for FAFI]
\label{th:imse}
Let $\delta \in (0, 1)$ and
\begin{equation}
	\beta_t = 2 \ln \frac{\pi^2 n t^2}{6 \delta}.
\end{equation}
We run Algorithm~\ref{alg:farthest_first} for a sample function $f \sim \mathrm{GP}(0, k(\cdot, \cdot) + \sigma^2 \delta (\vecX - \vecX'))$. 
Then, for it with the probability at least $1 - \delta$ for all $2 \le t \le n$ it holds that
\begin{equation}	
	E_t \le \beta_t \left(
	C \frac{2\Delta}{(t - 1)^{1 / \xDim} - 1}
	 + \sigma^2
	\right),
\end{equation}
where $\Delta := \mathrm{diam}(\mathcal{D}) = \max_{\vecX, \vecX' \in \mathcal{D}} \|\vecX - \vecX' \|_2$, and $C > 0$ is a constant independent of all other parameters. 
\end{theorem}
The theorem provides a bound for the procedure that is independent of $f(\vecX)$ and its observed values.

To connect uncertainty reduction to \emph{rank preservation}, consider two models with independent performance functions \(f_1, f_2 \sim {}\allowbreak{} \mathrm{GP}(0, k(\cdot, \cdot) + \sigma^2 \delta (\vecX - \vecX'))\). 
We define the pointwise gap
\begin{equation}
g(\vecX) := f_1(\vecX) - f_2(\vecX).
\end{equation}
Let $\hat f_{1, t},\hat f_{2, t}$ be the GP posterior means constructed after selecting $t$ points by Algorithm~\ref{alg:farthest_first}, and define $\mathrm{Err}_t$ as the fraction of sign errors for $\hat{g}_{t}(\vecX) = \hat f_{1, t}(\vecX) - \hat f_{2, t}(\vecX)$ if compared with $g(\vecX)$: there the estimated ordering disagrees with the true one.
The posterior variance $\sigma^2_{t}(\vecX)$ coincides for both posteriors if the kernel and the design points are the same. 
Then the following theorem holds.

\begin{theorem}[Expected ranking error bound]
\label{th:rank_error}
% (denoted $\mathcal{D}_{t-1}$),
Conditionally on the observations up to step $t - 1$: 
\begin{align*}
\mathbb{E}\!\left[\mathrm{Err}_t \mid \mathcal{D}_{t-1}\right]
 &= \frac{1}{n} \sum_{x \in \mathcal{D}}
 \mathbb{P}\!\left( g(x)\, \hat{g}_{t - 1}(x) < 0 \mid \mathcal{D}_{t-1}\right) \\
 &\le \frac{1}{n} \sum_{x \in \mathcal{D}}
\exp \left(-\frac{|\hat g_{t-1}(x)|^2}{4 \sigma^2_{t - 1}(x)}\right).
\end{align*}
Moreover, if $\sigma^2_{t-1}(x) \le S_t$ for all $x \in \mathcal{D}$, then
\begin{equation}
\mathbb{E}\!\left[\mathrm{Err}_t \mid \mathcal{D}_{t-1}\right]
\le \exp \left(-\frac{|\hat g_{\mathrm{abs}\min}^t|^2}{4 S_t}\right).
\end{equation}
with $S_t = C \frac{2 \Delta}{(t - 1)^{1/d} - 1}
+\sigma^2$, $\hat g_{\mathrm{abs}\min}^t := \min_{x \in \mathcal{D}} |\hat g_{t-1}(x)|$.

\end{theorem}

\paragraph{Theoretical results summary}
These statements formalize the intuition that farthest-first is beneficial when the chosen representation induces a geometry in which model performance varies smoothly: the % algorithm explicitly controls the covering radius (hence uncertainty) and 
yields a decreasing with $t$ upper bound on MSE and ranking errors as $t$ grows. 
In contrast, uniform random sampling does not control the covering radius in the worst case and therefore does not admit an analogous coverage-driven guarantee without additional assumptions on the distribution of datasets.
The second theorem also provides a natural way to quantify the ranking quality given the subset size and the minimal gap in model performance.

\section{Benchmarking Protocol}
\label{sec:bench_prot}

In this section, we describe an empirical evaluation protocol used to compare datasets' subset selection strategies in terms of preserving the global ranking of models formally defined in the previous section and the representation of datasets in the considered domains.

\subsection{Evaluation methodology}
\label{sec:evaluation_methodology}

We adopt a unified and reusable benchmarking protocol designed to evaluate
datasets subset selection strategies with respect to their ability to preserve
the model ranking induced by the full benchmark.
The protocol explicitly accounts for variability in dataset and model availability
and provides statistically robust estimates of rank-preservation performance for each \emph{(strategy, representation)} pair under variations of
(i) the subset size $k$,
(ii) the available dataset pool,
and (iii) the available model pool.

The evaluation procedure is organized into three stages.

\paragraph{Stage 1: Generation of evaluation trials.}
To assess robustness, we generate a set of $T$ independent evaluation trials
\begin{equation}
\mathcal{T} = \{t_1, t_2, \ldots, t_T\},
\end{equation}
where each trial corresponds to a randomized perturbation of the benchmarking setup.
We consider two complementary trial-generation scenarios:

\begin{itemize}
  \item \textbf{Dataset pool variation:} each trial $t$ consists of a randomly sampled subset
  $\mathcal{D}_t \subset \mathcal{D}$ with $|\mathcal{D}_t| = [\alpha D]$, sampled without replacement.
  \item \textbf{Model pool variation:} each trial $t$ consists of the full dataset collection
  $\mathcal{D}$ and a randomly sampled subset of models
  $\mathcal{M}_t \subset \mathcal{M}$ with $|\mathcal{M}_t| = [\alpha M]$, sampled without replacement.
\end{itemize}

The subsampling ratio $\alpha \in (0,1)$ is fixed in advance.

\paragraph{Stage 2: Subset selection and ranking construction.}
For each trial $t \in \mathcal{T}$ and a fixed subset size $k$, the considered selection strategy produces a datasets subset
$$
\mathcal{S}_t \subseteq \mathcal{D}_t \quad \text{or} \quad \mathcal{S}_t \subseteq \mathcal{D},
\qquad |\mathcal{S}_t| = k,
$$
depending on the evaluation scenario. Using the selected subset $\mathcal{S}_t$, we compute the vector of average model ranks $\mathbf{R}_{\mathcal{S}_t},$ as defined in Section~\ref{sec:problem_definition}. This vector represents the ranking induced by evaluating models only on the selected datasets.

As a reference, we always use the ranking induced by the full benchmark.
In the dataset pool variation scenario, the target ranking is given by $\mathbf{R}_{\mathcal{D}},$ computed over all datasets and models.
In the model pool variation scenario, the reference ranking is obtained by restricting
$\mathbf{R}_{\mathcal{D}}$ to the sampled model subset $\mathcal{M}_t$, with ranks
renormalized to the range $1, \ldots, |\mathcal{M}_t|$.

\paragraph{Stage 3: Metric computation and aggregation.}
For each trial $t$, we compute a ranking discrepancy metric between the subset-induced ranking and the corresponding reference ranking:
\begin{equation}
a_t =
\operatorname{Metric}\bigl(
\mathbf{R}_{\mathcal{S}_t},
\mathbf{R}_{\mathcal{D}}
\bigr).
\end{equation}

This yields a set of metric values
\begin{equation}
\mathcal{A}_T
= \bigl\{ a_t \mid t \in \mathcal{T} \bigr\}.
\end{equation}
We report the mean performance
\begin{equation}
\bar{A}_T = \frac{1}{T} \sum_{t \in \mathcal{T}} a_t,
\end{equation}
along with non-parametric confidence intervals computed using empirical quantiles.
Let $F^{-1}_{\mathcal{A}_T}(p)$ denote the $p$-quantile of $\mathcal{A}_T$.
Then a $(1-\beta)$ confidence interval is given by
\begin{equation}
\bigl[
F^{-1}_{\mathcal{A}_T}(\beta/2),
F^{-1}_{\mathcal{A}_T}(1-\beta/2)
\bigr].
\end{equation}

Repeating this procedure for all considered subset sizes $k$ and all
\emph{(selection strategy, dataset representation)} pairs yields a comprehensive and robust evaluation of ranking preservation performance.

\subsection{Dataset Representations}
\label{sec:representations}

Our subset selection strategies operate on a dataset-level embedding
$\vecX \in \mathbb{R}^p$ that characterizes each dataset independently of the evaluated benchmark models and is used solely for selecting a subset of datasets (cf. Eq. \ref{eq:x_p}). 
We focus on prior representations that do not require access to the benchmark evaluation outcomes. 
Concretely, we consider two families: (i) a priori meta-features computed directly from the dataset content, and (ii) landmarking/probe-based descriptors computed by evaluating a small, fixed set of lightweight probe learners that are not part of the benchmark model pool.

\paragraph{Time Series Classification (TSC)} In this domain, we use six dataset representation families:
\begin{itemize}
    \item \textbf{Simple Features.} 
    A low-dimensional feature description reflecting the basic statistics of a dataset: dataset size, class imbalance (Gini index and entropy), and a binary or multiclass classification indicator. These features provide a coarse representation.
    
    \item \textbf{TSFRESH / Catch22 / Summary / MiniRocket embeddings.} 
    For each dataset $d_i$ containing $n_i$ time series, we compute per-series feature vectors $\textbf{f}_{ij} \in \mathbb{R}^{p}$ and aggregate them into a dataset embedding via the mean:
    $$
    \textbf{x}_i = \frac{1}{n_i}\sum_{j=1}^{n_i} \textbf{f}_{ij}.
    $$

    \item \textbf{Landmarking representation.}
    Besides hand-crafted or extracted time-series features, we also use landmarking representations. The idea is simple: we characterize a dataset by how a small set of very cheap probe learners performs on it. These probes are not part of the benchmark model pool and are only used to obtain a lightweight signature of dataset difficulty and structure. Concretely, we run a fixed set of probes (decision stump, 1-NN, a linear classifier, and a shallow decision tree) using the same evaluation protocol as for the benchmark (cross-validation). From their results, we build a compact dataset-level vector that includes: (i) mean CV error of each probe, (ii) a simple generalization-gap signal (train--test gap for the linear probe), and (iii) stability indicators such as error and rank variability estimated via CV and bootstrap (e.g., variance of the probe errors and variance of the induced probe ranks).
    
\end{itemize}

\paragraph{Recommender Systems.}
For RecSys benchmarks, we represent each dataset using only properties of the user--item interaction matrix $R \in \mathbb{R}^{|U|\times|I|}$: size/shape measures (e.g., $|U|$, $|I|$, $|R|$, $|U|\cdot|I|$, $|U|/|I|$), sparsity/activity measures (density, ratings per user/item), inequality statistics (Gini), and distributional descriptors capturing popularity bias and long-tail effects, as described in \cite{DELDJOO2021102662}, and the extracted feature descriptions themselves were taken from the results reported in \cite{Shevchenko2024RecSysBenchmarking}. 
This yields an 18-dimensional prior representation for each dataset.

\paragraph{Natural Language Processing.}
We consider two lightweight, no-leakage dataset representations for NLP.
First, we construct a short structured natural-language description for
each dataset using a fixed template covering task family, input format,
retrieval or classification objective label, text granularity, language setup, domain, label or relevance structure, and special properties such as lexical mismatch or noisy text. 
These descriptions are encoded with
\texttt{bert-base-uncased}: we tokenize each description with maximum
length 256, mean-pool final-layer token embeddings over non-padding
tokens, and $\ell_2$-normalize the resulting 768-dimensional vector.
Second, we use simple structured features from the MTEB metadata
reported in Table~2 of~\cite{muennighoff2023mteb}, including
task/category, number of languages, train/dev/test sizes, and average
text lengths; categorical attributes are one-hot encoded and numerical
attributes are rescaled.

\begin{table*}[t]
  \centering
  \caption{\textbf{AUC aggregates over subset sizes $k=2,\dots,20$ for three domains.}
  We report mean $\pm$ standard deviation of AUC(MAE) $\downarrow$, AUC($\rho$) $\uparrow$, AUC(NDCG@5) $\uparrow$, AUC($\tau$)$\uparrow$, and AUC(MRR)$\uparrow$.
  Each strategy is shown with its best-performing dataset representation.}
  \label{tab:auc_combined}
  \begin{tabular}{lccccc}
    \toprule
    % Strategy &
     &
    AUC(MAE)$\downarrow$ &
    AUC($\rho$)$\uparrow$ &
    AUC(NDCG@5)$\uparrow$ &
    AUC($\tau$)$\uparrow$ & 
    AUC(MRR)$\uparrow$\\
    \midrule

    \multicolumn{5}{l}{\textbf{TSC}} \\
    \midrule
    Random 
    & 31.06 $\pm$ 1.73 
    & 16.61 $\pm$ 0.19 
    & 17.58 $\pm$0.09 
    & 14.24 $\pm$0.25
    & 11.39 $\pm$ 1.40\\
    Cosine Farthest-First + \texttt{Catch22}   
    & 27.07 $\pm$ 2.99 
    & \textbf{17.18} $\pm$ 0.33 
    & \underline{17.67} $\pm$ 0.11 
    & \textbf{15.11} $\pm$ 0.54
    & $11.21 \pm 3.36$ \\
    Euclidean Farthest-First + \texttt{MiniRocket}
    & \textbf{25.68} $\pm$ 4.04 
    & \underline{16.91} $\pm$ 0.36 
    & \textbf{17.70} $\pm$ 0.14 
    & \underline{14.77} $\pm$ 0.54 
    & $\mathbf{13.65} \pm 3.56$\\
    K-Means + \texttt{TSFRESH}
    & \underline{26.67} $\pm$ 2.79 
    & 16.82 $\pm$ 0.19 
    & 17.62 $\pm$ 0.09 
    & 14.66 $\pm$ 0.29
    & \underline{13.00} $\pm$ 2.31\\
    A-optimality + \texttt{Landmarking}
    & 28.12 $\pm$ 1.72 
    & 16.75 $\pm$0.16 
    & 17.52 $\pm$ 0.08 
    & 14.38 $\pm$ 0.26
    & $11.61 \pm 1.66$\\
    D-optimality + \texttt{Landmarking}
    & 28.93 $\pm$ 1.45 
    & 16.71 $\pm$0.15 
    & 17.55 $\pm$ 0.08
    & 14.28 $\pm$ 0.24
    & $12.56 \pm 1.67$\\

    \midrule
    \multicolumn{5}{l}{\textbf{RecSys}} \\
    \midrule
    Random
    & 9.52 $\pm$ 0.64 
    & \textbf{14.88} $\pm$ 0.66 
    & \underline{17.54} $\pm$ 0.12 
    & \underline{12.94} $\pm$ 0.72 
    & $16.80 \pm 0.80$\\
    Cosine Farthest-First
    & 8.92 $\pm$ 1.06 
    & \underline{14.84} $\pm$ 1.08 
    & 17.46 $\pm$ 0.19 
    & \textbf{13.01} $\pm$ 1.13 
    & $17.10 \pm 1.08$ \\
    Euclidean Farthest-First
    & 9.09 $\pm$ 1.14 
    & 14.40 $\pm$ 0.88 
    & 17.48 $\pm$ 0.15 
    & 12.57 $\pm$ 0.96 
    & $16.85 \pm 0.56$\\
    K-Means
    & \textbf{8.67} $\pm$ 0.91 
    & 14.61 $\pm$ 0.85 
    & \textbf{17.55} $\pm$ 0.13 
    & 12.71 $\pm$ 0.94 
    & $\mathbf{17.49} \pm 0.59$\\
    A-optimality
    & 8.88 $\pm$ 0.97 
    & 14.54 $\pm$ 0.78 
    & 17.51 $\pm$ 0.15 
    & 12.58 $\pm$ 0.87 
    & $17.39 \pm 0.59$ \\
    D-optimality
    & \underline{8.83} $\pm$ 0.99 
    & 14.51 $\pm$ 0.79 
    & 17.52 $\pm$ 0.14 
    & 12.56 $\pm$ 0.87 
    & \underline{17.43} $\pm$ 0.60\\

    \midrule
    \multicolumn{5}{l}{\textbf{NLP}} \\
    \midrule
    Random
    & 26.50 $\pm$ 2.06
    & 16.75 $\pm$ 0.20
    & 17.80 $\pm$ 0.05
    & 14.64 $\pm$ 0.29
    & \underline{11.51} $\pm$ 1.88 \\
    Cosine Farthest-First + \texttt{BERT-features}
    & 22.41 $\pm$ 2.56
    & \textbf{17.20} $\pm$ 0.24
    & \underline{17.89} $\pm$ 0.03
    & \textbf{15.33} $\pm$ 0.41
    & 7.88 $\pm$ 2.92 \\
    Euclidean Farthest-First + \texttt{BERT-features}
    & \underline{21.49} $\pm$ 3.28
    & \underline{17.17} $\pm$ 0.23
    & \underline{17.89} $\pm$ 0.04
    & \underline{15.17} $\pm$ 0.46
    & 8.60 $\pm$ 3.03 \\
    K-Means + \texttt{BERT-features}
    & \textbf{21.05} $\pm$ 2.18
    & 16.93 $\pm$ 0.27
    & \textbf{17.91} $\pm$ 0.03
    & 15.05 $\pm$ 0.40
    & \textbf{12.38} $\pm$ 2.16 \\
    A-optimality + \texttt{Simple Features}
    & 26.27 $\pm$ 3.02
    & 16.61 $\pm$ 0.27
    & \underline{17.86} $\pm$ 0.04
    & 14.22 $\pm$ 0.44
    & 5.47 $\pm$ 2.44 \\
    D-optimality + \texttt{Simple Features}
    & 27.11 $\pm$ 3.37
    & 16.51 $\pm$ 0.31
    & 17.87 $\pm$ 0.04
    & 14.09 $\pm$ 0.46
    & 4.86 $\pm$ 2.17 \\
    
    \bottomrule
  \end{tabular}
\end{table*}

\begin{figure}[t]
    \centering
    \footnotesize

    \begin{minipage}[t]{0.48\linewidth}
        \centering
        \includegraphics[width=\linewidth]{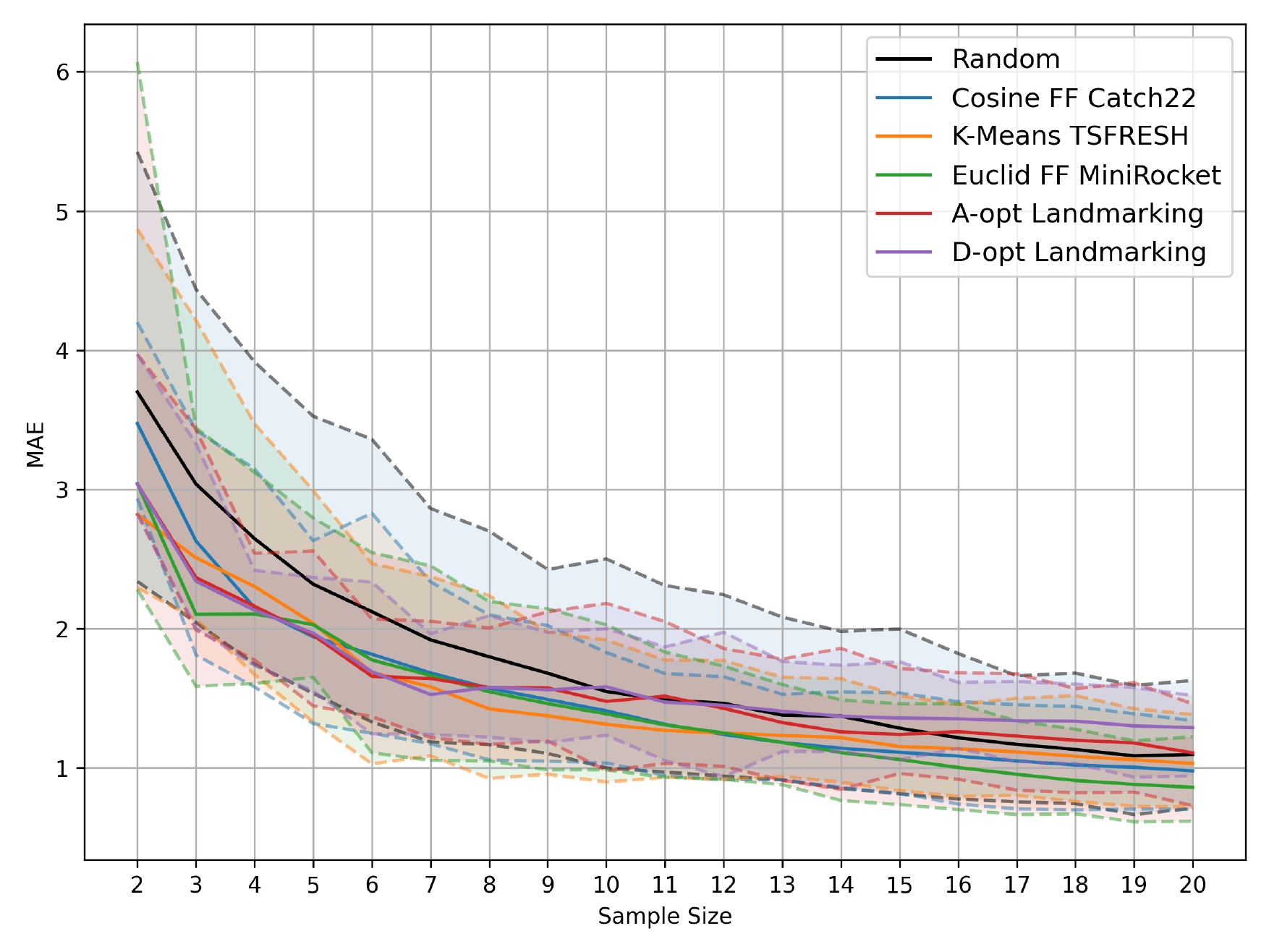}
        \caption*{TSC: MAE $\downarrow$}
    \end{minipage}\hfill
    \begin{minipage}[t]{0.48\linewidth}
        \centering
        \includegraphics[width=\linewidth]{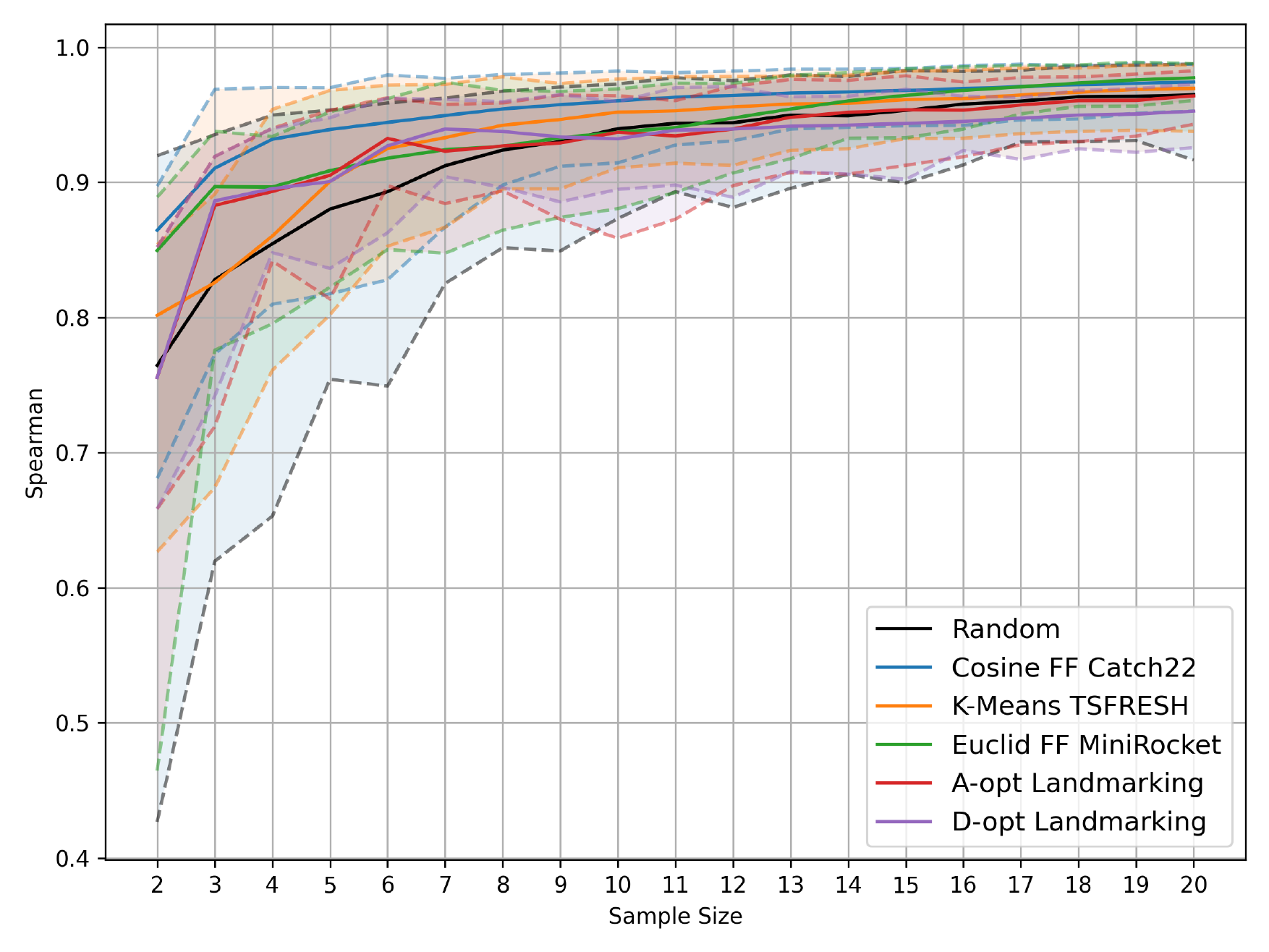}
        \caption*{TSC: Spearman $\uparrow$}
    \end{minipage}

    \vspace{0.35em}

    \begin{minipage}[t]{0.48\linewidth}
        \centering
        \includegraphics[width=\linewidth]{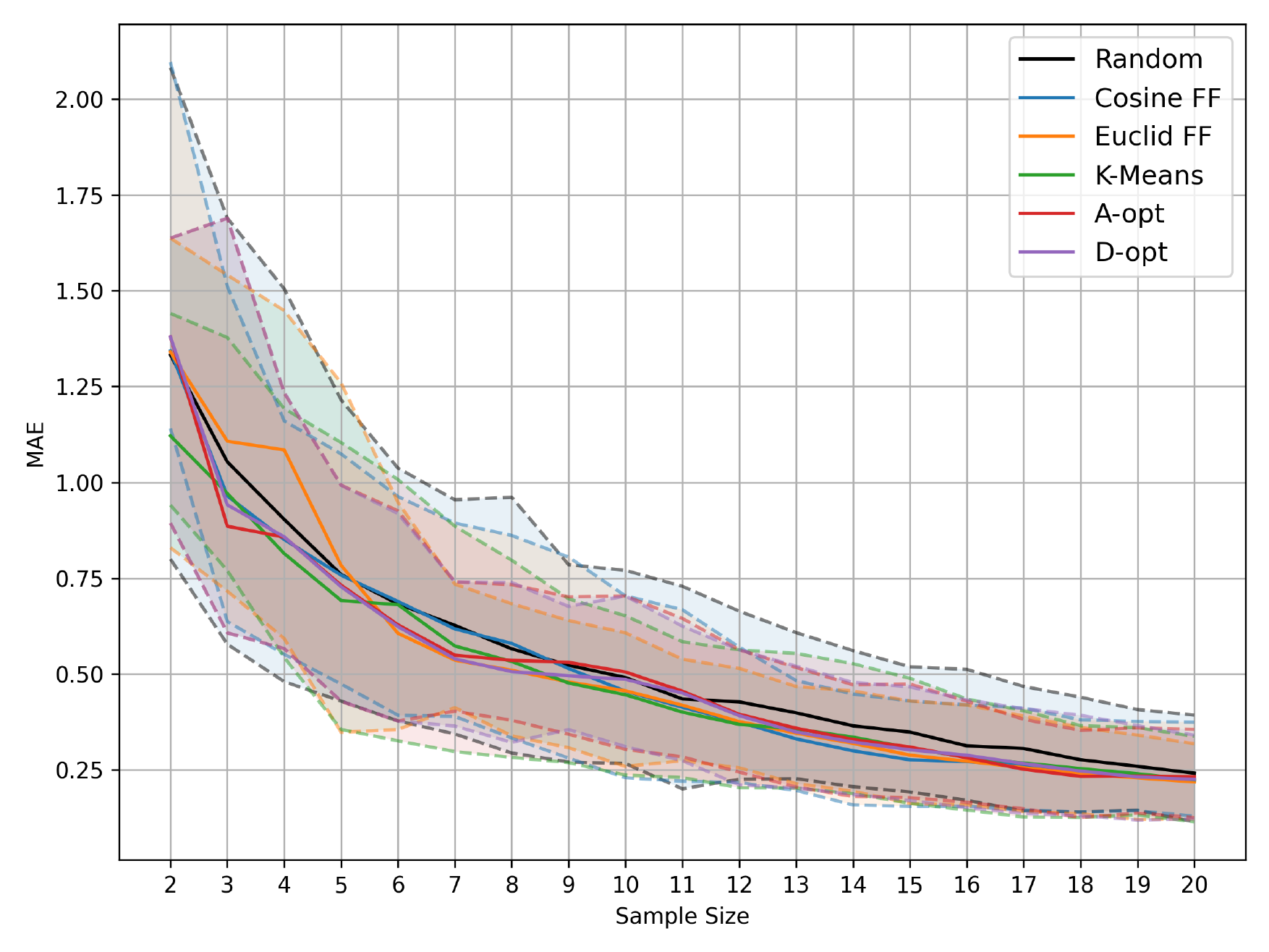}
        \caption*{RecSys: MAE $\downarrow$}
    \end{minipage}\hfill
    \begin{minipage}[t]{0.48\linewidth}
        \centering
        \includegraphics[width=\linewidth]{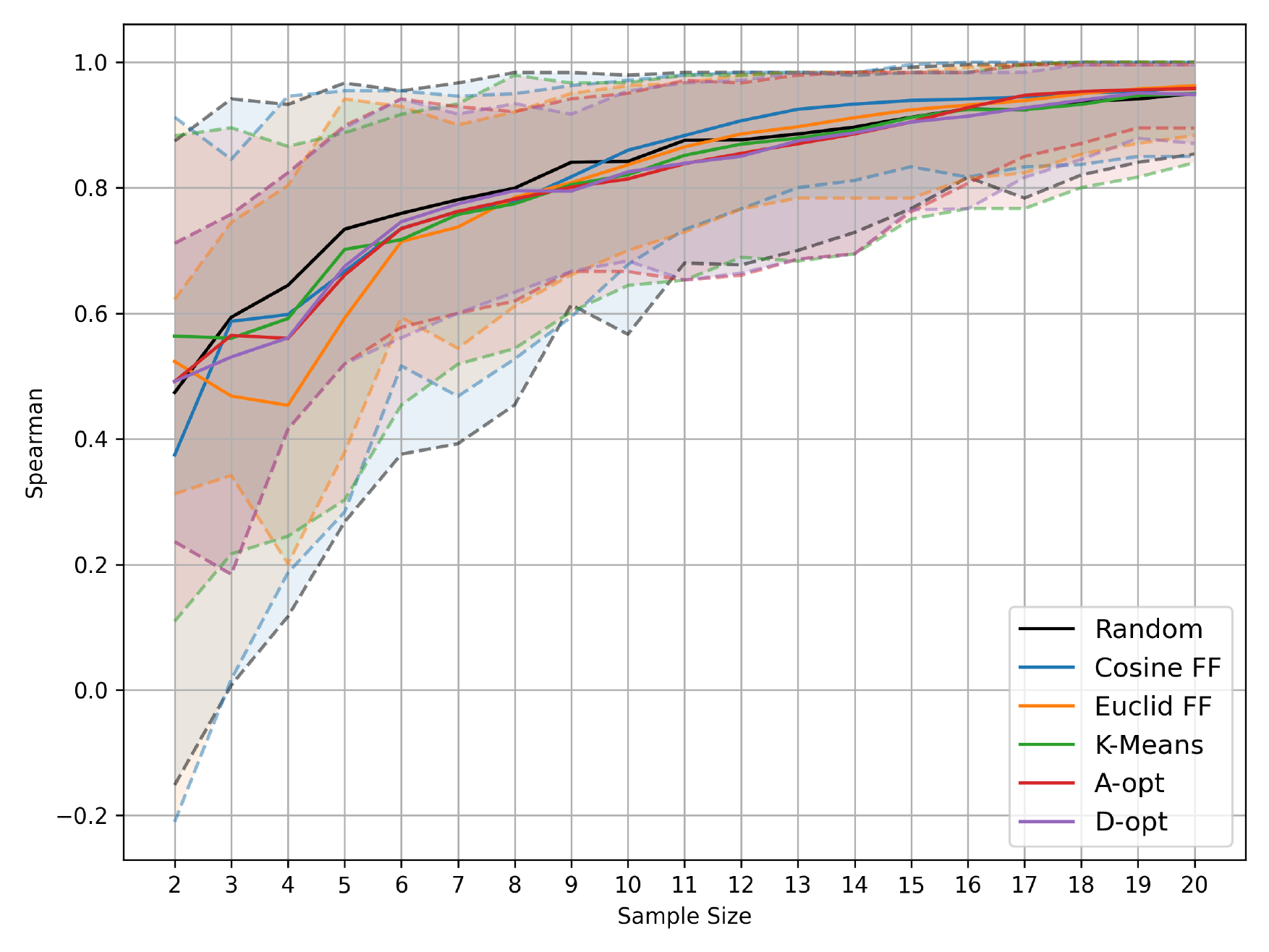}
        \caption*{RecSys: Spearman $\uparrow$}
    \end{minipage}

    \vspace{0.35em}

    \begin{minipage}[t]{0.48\linewidth}
        \centering
        \includegraphics[width=\linewidth]{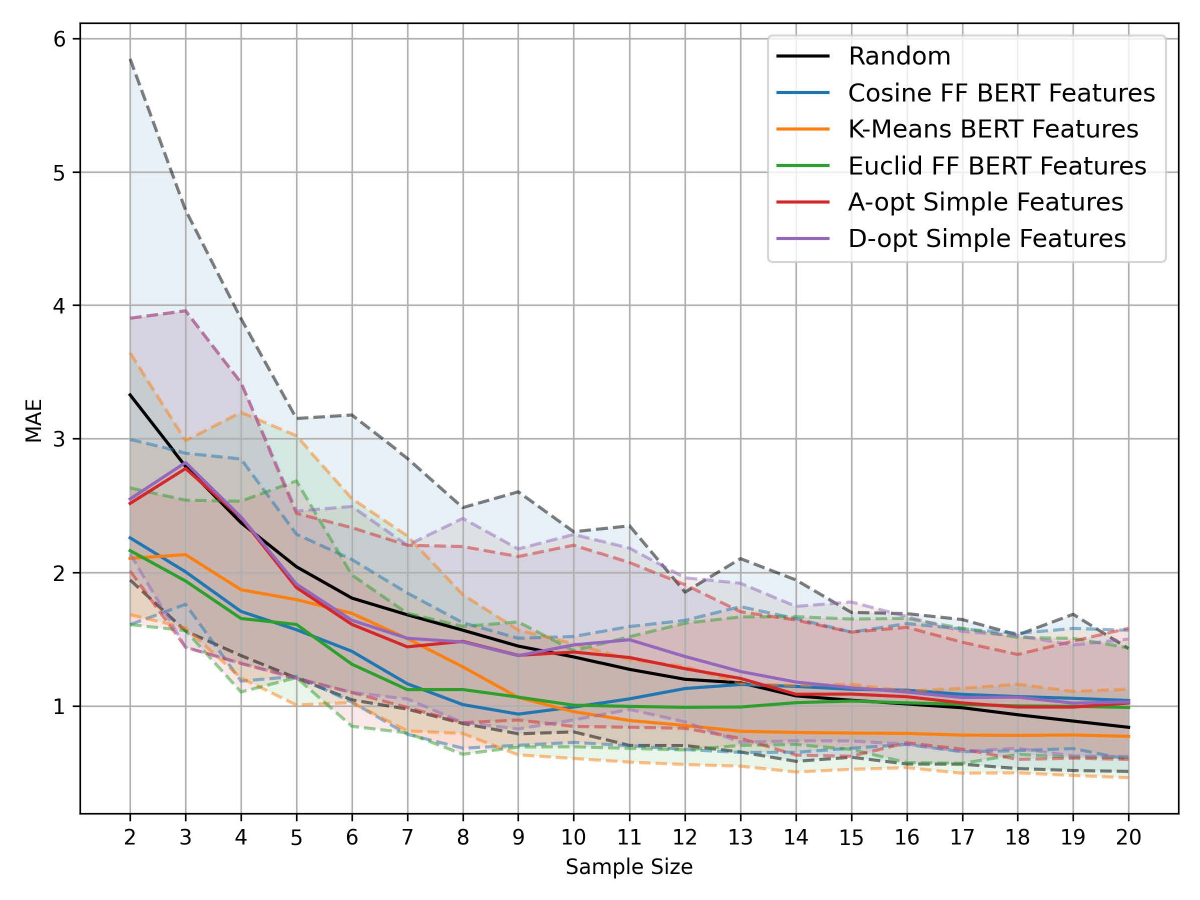}
        \caption*{NLP: MAE $\downarrow$}
    \end{minipage}\hfill   
    \begin{minipage}[t]{0.48\linewidth}
        \centering
        \includegraphics[width=\linewidth]{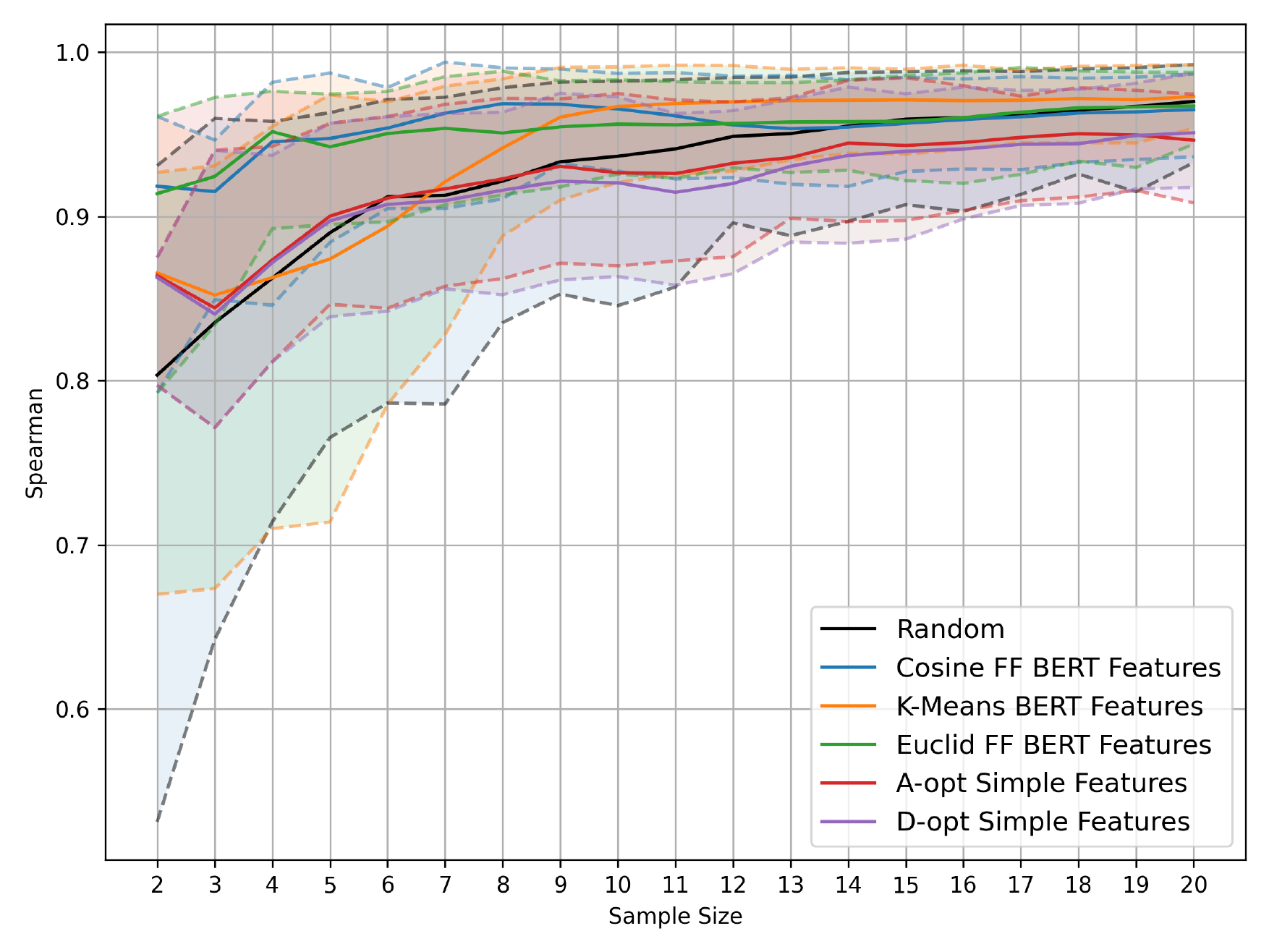}
        \caption*{NLP: Spearman $\uparrow$}
    \end{minipage}

    \caption{
    Rank preservation under dataset subset selection in TSC, RecSys, and NLP.
    We compare selection strategies as a function of subset size $k$ using
    rank MAE ($\downarrow$) and Spearman correlation ($\uparrow$) with respect
    to the full-benchmark ranking. Curves show means and 95\% confidence
    intervals over $T=200$ bootstrap trials with sampling fraction $\alpha=0.8$.
    % Better to view in zoom.
    }
    \label{fig:best_pairs_all_domains}
\end{figure}

\section{Results}
\label{sec:main_results}

\paragraph{Domains} 
We evaluate datasets subset selection strategies in two domains:
Time Series Classification (TSC) and Recommender Systems (RecSys),
and additionally include a supplementary NLP benchmark derived from
MTEB~\cite{muennighoff2023mteb}. 
For experiments, we use a complete score matrix obtained by evaluating each model on every dataset under a fixed, consistent evaluation protocol.

For TSC, we rely on publicly available benchmark results collected from the community-maintained repository~\footnote{timeseriesclassification.com}.
It aggregates evaluation outcomes from multiple large-scale experimental studies. 
In particular, we use results from recent bake-off benchmarks,
covering 112 datasets and 35 classification methods,
evaluated under standardized cross-validation protocols.

For the RecSys domain, we use the benchmark introduced by
Shevchenko et al.~\cite{Shevchenko2024RecSysBenchmarking},
which evaluates collaborative filtering models that utilize past user interactions for next item recommendation across 30 public
datasets under a unified offline evaluation protocol.
We consider 9 representative models and construct a complete
$30 \times 9$ score matrix using the reported results.
Following the original benchmark, \textit{NDCG@10} is used as the
primary evaluation and comparison metric.

For NLP, we use the public MTEB benchmark~\cite{muennighoff2023mteb}. The model--dataset performance
matrix is from Table~11 of~\cite{muennighoff2023mteb}.
After aligning it with available dataset-level representations, we obtain 57 NLP datasets/tasks and 31 text embedding models. 
As in the other domains,
raw benchmark scores are used only to induce model rankings: for each
dataset, models are ranked by their reported score, and subset quality
is measured by how well the selected datasets preserve the full
MTEB-derived model ranking.

\paragraph{Evaluation protocol}
We consider the stability of rankings under dataset and method subsampling.
For dataset subsampling, for each subset size $k \in {2, \dots, 20}$, we sample $T = 200$ training splits of datasets, each split containing a fraction $\alpha$ of all datasets, with the default $\alpha = 0.8$.
Given a training split, each selection method chooses a subset $\mathcal{S}_t$ of size $k$ from the training pool; we then measure how well the model rankings induced by $\mathcal{S}_t$ match the global model ranking computed on the full dataset pool (cf. Section~\ref{sec:problem_definition}).
We report mean curves and $95\%$ confidence intervals across splits for MAE and correlation for the rankings described below in detail.
For model-pool variation, the average performance is close to the result of a single run of the algorithm on the full dataset, and the main effect is reflected in the uncertainty bands. Therefore, we focus the reported figures and AUC tables on dataset-pool variation.

\paragraph{Aggregated metrics}
\label{sec:auc_metrics}
To summarize performance over the entire range of subset sizes, we compute the area under the curve (AUC) with respect to $k$ using the trapezoidal rule.
For a curve $y(k)$ evaluated at $k = 2,\dots,k_{\max}$ we define:
\begin{equation}
\mathrm{AUC}(y) = \sum_{k = 2}^{k_{\max} - 1} \frac{y(k) + y(k + 1)}{2} \cdot 1.
\label{eq:auc_trap}
\end{equation}
We constrain $k_{\max}$ to $20$ for the benchmarking to remain efficient, so the measured metrics make sense for the considered problem statement.
This aggregation over subset sizes resembles AUC-style evaluation and corresponds, up to a constant scaling and boundary corrections, to a mean regret.
Five complementary metrics from Section~\ref{sec:bench_sett} are considered: $\mathrm{MAE}(k)$, Spearman correlation $\rho(k)$, Kendall’s rank correlation $\tau(k)$, $\mathrm{NDCG@5}(k)$, and $\mathrm{MRR}(k)$.

\subsection{Main results}

Table~\ref{tab:auc_combined} reports AUC aggregates over $k = 2, \dots, 20$.  The presented aggregates over different metrics dynamics during the addition of datasets to a benchmark provide a broad picture of the quality of the used dynamic subset selection approaches.
Figure~\ref{fig:best_pairs_all_domains} shows how ranking preservation improves as subset size increases.

\paragraph{Time Series Classification}
\label{sec:result_tsc}
Across methods, we observe the largest gains at small $k$ (typically $k \le 10$), followed by diminishing returns as $k$ approaches $20$.

We compare diversity-based and clustering-based selection with Random sampling.
Diversity-oriented methods (Farthest-First with cosine or Euclidean distance) and K-Means generally outperform Random when the dataset representation captures ranking-relevant structure.
Oracle-style representations derived from a posteriori information (e.g., rank features) provide an upper bound on achievable performance for geometry-based selection.

\paragraph{Recommender Systems}
\label{sec:results_recsys}
Overall, the improvements over Random are weak and inconsistent. Although the average MAE decreases with increasing subset size~$k$, the differences between the selection strategies are small and fall within uncertainty bands. The k-means and diversity-based farthest-first methods achieve slightly lower average MAEs than Random, but the gains remain modest. 
Regarding rank order preservation, as measured by the remaining metrics, no method consistently outperforms Random, with in some cases Random demonstrating the best overall results. 
So, in RecSys, the available dataset priors provide limited information about the models' ranking structure, making geometry-based subset selection significantly less robust.

\paragraph{Natural Language Processing}
\label{sec:nlp_recsys}
The results are closer to TSC than to RecSys: when description-based
BERT representations are used, the geometry-based selection methods preserve the full-benchmark ranking substantially better than random sampling, especially for small subset sizes. 
Structured metadata features provide a weaker but still informative signal. 
This confirms that the protocol itself is domain-agnostic, while the practical gains depend on whether the available dataset representation captures ranking-relevant task differences.

\subsection{Practical subset-size guidance}

Beyond comparing AUC aggregates, we summarize how many datasets are needed to reach useful levels of rank preservation. 
For each domain and metric, we record the smallest subset size $k^\star$ at which a method--representation pair reaches a target threshold. 
We report best/median/worst $k^\star$ across successful configurations using both the mean metric curve and a conservative criterion based on the worst side of the confidence interval. 
Table~\ref{tab:practical_guidance} uses reasonable MAE $\leq 1.5$ and Spearman correlation $\rho \geq 0.90$ as targets.
For MAE $\leq 1.5$, the required subset sizes are small in RecSys, with conservative $k^\star$ being $5$, 
while TSC and NLP require larger subsets, with conservative median values of $18$ and $17$, respectively. 
For Spearman $\rho \geq 0.90$, the mean curves reach the target with median $k^\star = 5$ for TSC, $15$ for RecSys, and $5.5$ for NLP, but the conservative criterion is substantially stricter: the corresponding medians increase to $11$ and $12$ for TSC and NLP, and no RecSys configuration reaches the target for $k \leq 20$.

\begin{table}[t]
\centering
\footnotesize
\setlength{\tabcolsep}{3.5pt}
\renewcommand{\arraystretch}{1.08}
\caption{Practical subset-size guidance. For each domain and metric, we
report best/median/worst $k^\star$ across successful
method--representation pairs. The mean column uses the average metric
curve; the conservative column uses the worst side of the confidence
interval. A dash means that no configuration reaches the target for
$k \leq 20$.}
\label{tab:practical_guidance}
\begin{tabular}{llll}
\toprule
Domain & Target & Mean $k^\star$ & Cons. $k^\star$ \\
\midrule
\multirow{2}{*}{TSC} 
& MAE $\leq 1.5$ & 7 / 11 / 18 & 14 / 18 / 20 \\ 
& $\rho \geq 0.90$ & 3 / 5 / 9 & 7 / 11 / 18 \\ 
\multirow{2}{*}{RecSys} 
& MAE $\leq 1.5$ & 2 / 2 / 2 & 2 / 4 / 5 \\ 
& $\rho \geq 0.90$ & 12 / 15 / 15 & -- \\ 
\multirow{2}{*}{NLP} 
& MAE $\leq 1.5$ & 6 / 7.5 / 9 & 10 / 17 / 20 \\ 
& $\rho \geq 0.90$ & 2 / 5.5 / 7 & 6 / 12 / 17 \\
\bottomrule
\end{tabular}
\end{table}

\subsection{Statistical significance of selection strategies}

To complement aggregate AUC comparisons, we test whether the observed differences between selection strategies are statistically reliable.
For each domain and metric, we compute trial-wise AUC values over subset sizes $k = 2,\ldots,20$, select the best-performing method--representation pair, and compare it against the remaining methods using a one-sided paired Wilcoxon signed-rank test. 
We correct $p$-values within each domain--metric block using the Holm procedure.
Thus, our full procedure for testing the statistical significance of gains follows that of~\cite{Ismail_Fawaz_2019} with detailed discussion available in~\cite{demsar2006statistical}.
Table~\ref{tab:significance_summary} summarizes the results for the two main metrics used in the practical guidance analysis.

\begin{table}[t]
\centering
\footnotesize
\setlength{\tabcolsep}{3.0pt}
\renewcommand{\arraystretch}{1.08}
\caption{Compact statistical comparison of selection strategies. 
For each domain and metric, we report the best method, its AUC gain over Random, the Holm-corrected $p$-value for the comparison with Random,
and the closest statistically indistinguishable competitor, if any.
For MAE, gain is computed as Random AUC minus best AUC; for Spearman,
as best AUC minus Random AUC.}
\label{tab:significance_summary}
\begin{tabular}{@{}llcccc@{}}
\toprule
Domain & Metric & Best & Gain & $p_{\mathrm{Holm}}$ & Near-top n.s. \\
\midrule
\multirow{2}{*}{TSC} & MAE        & E-FAFI        & 5.38 & $<10^{-26}$ & -- \\
 & $\rho$     & C-FAFI        & 0.57 & $<10^{-29}$ & -- \\
\multirow{2}{*}{RecSys}  & MAE        & K-Means     & 0.85 & $<10^{-19}$ & D-opt \\
& $\rho$     & Random      & --   & --          & C-FAFI \\
\multirow{2}{*}{NLP}    & MAE        & K-Means & 5.71 & $<10^{-32}$ & E-FAFI \\
   & $\rho$     & C-FAFI & 0.50 & $<10^{-29}$ & E-FAFI \\
\bottomrule
\end{tabular}
\end{table}

The significance results support the main conclusion that representation quality determines the value of non-random subset selection, with most gains statistically significant ($p$-values $< 0.01$). 
In TSC, representation-based selectors strongly outperform Random for both MAE and Spearman, with gains of $5.38$ and $0.57$ and highly significant Holm-corrected tests. 
NLP shows the same pattern, although the best methods are less sharply separated: K-Means is best for MAE ($5.71$), C-FAFI is best for Spearman ($0.50$), and E-FAFI is statistically indistinguishable from the top method in both metrics. 
RecSys is more metric-dependent: K-Means significantly improves MAE by $0.85$, but Random remains best for Spearman, with C-FAFI forming the closest non-significantly different competitor. 
Thus, non-random selection is reliable with high-fidelity rank occurring when the representation is sufficiently informative.

\subsection{Synthetic Oracle-vs-Broken Representation}

To isolate the role of representation quality, we construct a controlled
synthetic benchmark at the TSC scale. 
We used $N=112$ datasets, $M=35$ models, and $F=30$ folds. Each dataset $i$ and model $m$ is assigned a latent vector $z_i,u_m \in \mathbb{R}^{d_{\rm lat}}$, with $d_{\rm lat}=5$ and $\textbf{z}_i,\textbf{u}_m \sim \mathcal{N}(0,I)$. 
The latent error of model $m$ on dataset $i$ is defined by the squared distance $e_{i,m}=\|z_i-u_m\|_2^2$.
Fold-level scores are then generated as $s^{(f)}_{i,m}=e_{i,m}+\epsilon_i+\epsilon_{i,f}$, where
$\epsilon_i\sim\mathcal{N}(0,\sigma_{\rm ds}^2)$,
$\epsilon_{i,f}\sim\mathcal{N}(0,\sigma_{\rm fold}^2)$,
$\sigma_{\rm ds}=0.02$, and $\sigma_{\rm fold}=0.03$. 
We compute fold-wise ranks and mean model ranks using the same pipeline as in the main experiments.

We then compare two $d_{\rm rep}=64$-dimensional dataset
representations. 
Both are generated from the same latent dataset
geometry, but differ in how much of that geometry is visible to the
selection method. 
Let $A\in\mathbb{R}^{d_{\rm lat}\times d_{\rm rep}}$ be a random projection
with i.i.d. Gaussian entries scaled by $1/\sqrt{d_{\rm lat}}$. The
Oracle representation is $\textbf{x}^{\rm oracle}_i=\textbf{z}_iA+\boldsymbol{\eta}_i$, where
$\boldsymbol{\eta}_i\sim\mathcal{N}(0, \sigma^2 I)$, so distances in representation
space closely reflect the rank-generating geometry. 
In experiments $\sigma^2 = 0.01^2$.
The Broken representation is $\textbf{x}^{\rm broken}_i=\lambda \textbf{z}_iA+\boldsymbol{\xi}_i$, where $\lambda=0.15$ and $\boldsymbol{\xi}_i\sim\mathcal{N}(0,I)$; here the ranking-relevant signal is dominated by unrelated noise.

\begin{table}[t]
\centering
\footnotesize
\setlength{\tabcolsep}{3.5pt}
\renewcommand{\arraystretch}{1.08}
\caption{Synthetic representation-quality ablation. Values are AUC
aggregates over $k=2,\ldots,20$ reported as mean $\pm$ std over 100 seeds.}
\label{tab:synthetic_oracle_broken}
\begin{tabular}{llcc}
\toprule
Regime & Method & AUC(MAE)$\downarrow$ & AUC($\rho$)$\uparrow$ \\
\midrule
\multirow{2}{*}{Oracle}  & Random    & $37.95 \pm 2.79$ & $16.25 \pm 0.26$ \\
& Cosine FAFI & $20.23 \pm 2.39$ & $17.38 \pm 0.16$ \\

\multirow{2}{*}{Broken} & Random    & $39.40 \pm 3.18$ & $15.83 \pm 0.51$ \\
 & Cosine FAFI & $36.94 \pm 7.32$ & $16.04 \pm 0.80$ \\
\bottomrule
\end{tabular}
\end{table}

For each of 100 random seeds, we regenerate the full synthetic
benchmark and apply the same subset-selection protocol as in
Section~\ref{sec:main_results}, aggregating AUC values over $k=2,\ldots,20$. 
The results are reported in Table~\ref{tab:synthetic_oracle_broken}. 
In the Oracle regime, Cosine Farthest-First substantially improves over Random, reducing AUC(MAE) and increasing AUC($\rho$). 
In the Broken regime, the gap nearly disappears: Cosine Farthest-First reaches $36.94$ AUC(MAE) compared with $39.40$ for Random, and only marginally improves AUC($\rho$).
Thus, geometric selection helps when representation distances align
with model-ranking differences, but provides little advantage when
the representation is weakly aligned with the rank-generating structure.

\section{Conclusions and discussion}

Benchmarks across multiple domains span tens to hundreds of datasets, yet selecting a subset of tasks for faster evaluation remains largely heuristic and rarely justified by how well it preserves global model rankings. 
We frame dataset subset selection as a rank-preservation problem and introduce a unified protocol for systematically comparing selection strategies. 
The protocol includes bootstrap aggregation, which yields confidence intervals and enables principled statistical comparison of subset selection methods across subset sizes.

In time-series classification, when the dataset descriptors induce a meaningful geometry, greedy farthest-first selection consistently preserves rankings
better than random subsets. 
The supplementary NLP experiment shows a similar pattern when structured task descriptions are embedded with BERT,
suggesting that inexpensive semantic descriptions can already provide a useful geometry for benchmark construction. 
For recommender systems, in contrast, the advantage of sophisticated selection procedures is small, indicating that commonly available dataset features fail to capture the factors that drive model differentiation.
We also provide theoretical bounds for farthest-first selection, showing that improved geometric coverage is accompanied by lower approximation and pairwise ranking errors as the subset grows. 
To our knowledge, there has been little theoretical analysis of rank-preservation quality in adaptive design-of-experiments settings, and we derive bounds that explicitly link geometric coverage to ranking errors. 

The subset-size and significance analyses turn these observations into practical guidance: small subsets can often preserve correlation-based rankings in TSC and NLP, whereas stricter MAE targets or conservative confidence criteria require larger subsets.

Overall, this work provides a practical framework for rank-preserving benchmark reduction. 
Instead of treating benchmark subset choice as an informal or convenience-driven decision, we pose it as a measurable dataset selection problem with explicit rank-preservation metrics and statistical significance comparisons. 
The resulting protocol can be reused in a multi-dataset benchmark with more than $50$ tasks where evaluation costs are substantial and meaningful dataset-level representations can be constructed, with our experiments confirming the evidence from time series classification and NLP domains. 

\section{Acknowledgments}

The work was supported by the grant for research centers in the field of AI provided by the Ministry of Economic Development of the Russian Federation in accordance with the agreement \newline 000000C313925P4F0002 and the agreement №139-10-2025-033.

% \FloatBarrier
\bibliographystyle{ACM-Reference-Format}
\bibliography{sample-base}

\appendix

\section{Proofs of the statements in Section \ref{sec:theory}}

\subsection{Notation and preliminary remarks}

This appendix provides the auxiliary statements and proofs supporting
Section~\ref{sec:theory}, in particular Theorems~\ref{th:imse} and~\ref{th:rank_error}. All objects defined in Section~\ref{sec:theory}---the dataset pool $\mathcal{D}$, the kernel $k$, the dissimilarity $d_k$, the farthest-first sequence $(x_t)_{t\ge 2}$ with selected sets $\mathcal{S}_t$, the covering radius $r_t$, the GP model for $f$ with noise variance $\sigma^2$, and the discrete IMSE $E_t$---are used here without redefinition.
We also define the \emph{covering radius} $r_t$:
\[
r_t := \max_{\vecX \in \mathcal{D} \setminus \mathcal{S}_t} \min_{s \in \mathcal{S}_t} d_k(\vecX, \vecX'), 
\]
with 
$r_0 = \max_{\vecX, \vecX' \in \mathcal{D}} d_k(\vecX, \textbf{s})$.

For compactness in the proofs, we introduce the standard matrix notation for the GP posterior after $t$ observations at the selected points $S_{t} = \{\vecX_1, \dots, \vecX_{t}\}$:
\begin{gather}
K_{t} := \bigl[k(\vecX_i, \vecX_j)\bigr]_{i, j = 1}^{t}, \quad
\vecK_{t}(\vecX) := \bigl(k(\vecX_1, \vecX), \dots, k(\vecX_{t}, \vecX)\bigr)^\top, \\
\quad
\vecY_{t} := (y_1, \dots, y_{t})^\top, \text{where } y_i = f(\vecX_i).
\end{gather}
Whenever needed, we will use the GP posterior identities under $k(\vecX, \vecX) = 1$:
\begin{gather}
\label{eq:gp_posterior_app}
\mu_{t}(\vecX)
= \vecK_{t}(\vecX)^\top \bigl(K_{t} + \sigma^2 I\bigr)^{-1} \vecY_{t}, \\
\quad
\sigma^2_{t}(\vecX)
= 1 - \vecK_{t}(\vecX)^\top \bigl(K_{t} + \sigma^2 I\bigr)^{-1} \vecK_{t}(\vecX). 
\end{gather}

To relate posterior uncertainty to geometric coverage, we will repeatedly use the maximum kernel similarity between $\vecX$ and the selected set~$\mathcal{S}_{t}$:
\begin{equation}
\label{eq:m_t_def_app}
m_{t}(\vecX) := \max_{j \le t} k(\vecX, \vecX_j), \quad x\in\mathcal{D},
\end{equation}

\subsection{Preliminary Lemmas.}
The proof of Theorem~\ref{th:imse} combines: (i) an upper bound on $\sigma^2_{t-1}(\vecX)$ in terms of $m_{t-1}(\vecX)$ and $\sigma^2$, (ii) the uniform GP concentration inequality with parameter $\beta_t$ from Section~\ref{sec:theory}, and (iii) a covering argument for the farthest-first traversal. 
Theorem~\ref{th:rank_error} then follows by applying a Gaussian tail bound to the posterior error in the estimated pairwise gaps.

\begin{lemma}[Upper bound on the GP posterior variance]\label{lem:upper-variance}
Assume a zero-mean Gaussian process prior with covariance function $k(\cdot,\cdot)$ satisfying $k(\vecX, \vecX) = 1$, with additional observation noise variance $\sigma^2 > 0$. 
Then the posterior variance at any point $\vecX \in \mathcal{D}$ satisfies
\begin{equation}
\sigma_{t-1}^2(\vecX) \le 2\bigl(1 - m_{t-1}(\vecX)\bigr) + \sigma^2.
\end{equation}
\end{lemma}

\begin{proof}
Define $K := K_{t-1} + \sigma^2 I$, $\vecK = \vecK_{t-1} (\vecX)$.
The posterior variance of the Gaussian process at $\vecX$ is given by $\sigma_{t-1}^2(\vecX) = k(\vecX, \vecX) - \vecK^\top K^{-1} \vecK.$
Consider the quadratic function
\begin{equation}
F(\omega) := k(\vecX, \vecX) - 2 \omega^\top \vecK + \omega^\top K \omega,
\quad \omega \in \mathbb{R}^{t-1}.
\end{equation}
Its gradient is $\nabla F(\omega) = 2K\omega - 2\vecK$, which vanishes at $\omega^\star = K^{-1} \vecK$. 
Since the Hessian $\nabla^2 F(\omega) = 2K$ is positive semidefinite, $\omega^\star$ is a global minimizer, and
\begin{equation}
\min_{\omega} F(\omega) = F(\omega^\star) = \sigma_{t-1}^2(\vecX).
\end{equation}
Evaluating $F(\omega)$ at the standard basis vector $e_j$ yields $F(e_j) = k(\vecX, \vecX) + K_{jj} - 2 k(\vecX, \vecX_j).$
Using $k(\vecX, \vecX) = 1$ and $K_{jj} = k(\vecX_j, \vecX_j) + \sigma^2 = 1 + \sigma^2$,
we obtain $F(e_j) = 2\bigl(1 - k(\vecX, \vecX_j)\bigr) + \sigma^2.$

Since $\sigma_{t-1}^2(\vecX)$ is the minimum of $F(\omega)$, it follows that for any 
$j \le t - 1$:
\begin{equation}
\sigma_{t-1}^2(\vecX)
\le
2\bigl(1 - k(\vecX, \vecX_j)\bigr) + \sigma^2.
\end{equation}
Taking the minimum over $j$ yields the claimed bound.
\end{proof}

\begin{lemma}[Concentration of a Gaussian Process on a Finite Set, 
Lemma~5.1 in Appendix~I of \cite{Srinivas2010Gaussian}]
\label{lem:concentration}
Let $\delta \in (0, 1)$ and define $\beta_t = 2 \ln \frac{|\mathcal{D}|\,\pi_t}{\delta}$, where $\sum_{t=1}^\infty \pi_t^{-1} = 1,\, \pi_t > 0.$
Then, with probability at least $1 - \delta$, for all $t \ge 1$ and all $\vecX \in \mathcal{D}$, the following holds:
\begin{equation}
\bigl|f(\vecX) - \mu_{t-1}(\vecX) \bigr|
\le
\sqrt{\beta_t} \sigma_{t-1}(\vecX).
\end{equation}
\end{lemma}

Combining the result of Lemma~\ref{lem:upper-variance} ($\sigma^2_{t-1}(\vecX) \le 2(1 - m_{t-1}(\vecX)) + \sigma^2$) and the inequality from Lemma~\ref{lem:concentration}, we obtain that, with probability at least $1 - \delta$,
\begin{equation}
(f(\vecX) - \mu_{t-1}(\vecX))^2
\le
2 \beta_t \bigl(1 - m_{t-1}(\vecX)\bigr) + \beta_t \sigma^2.
\label{eq:sq_error_final}
\end{equation}

\begin{lemma}[Dimension-dependent covering bound]
\label{lem:covering-radius-dim}
Let $\mathcal{D} \subset \mathbb{R}^p$ be a finite set, and define its diameter as $\Delta := \mathrm{diam}(\mathcal{D}) = \max_{\vecX, \vecX' \in \mathcal{D}} \|\vecX - \vecX'\|_2 > 0$. Let the points $\vecX_1, \dots, \vecX_t$ ($2 \le t \le |\mathcal{D}|$) be selected according to the greedy
farthest-first rule~\eqref{eq:fufi} with $\mathcal{S}_{t-1} = \{\vecX_1, \dots, \vecX_{t - 1}\}$.
Define the Euclidean covering radius
\begin{equation}
r_t^e =
\max_{\vecX \in \mathcal{D} \setminus \mathcal{S}_t}
\min_{\vecS \in \mathcal{S}_t} \|\vecX - \vecS\|_2.
\end{equation}
Then, for all $2 \le t \le |\mathcal{D}|$, the following bound holds:
\begin{equation}
r_t^e \le \frac{2\,\Delta}{t^{1/p} - 1}.
\end{equation}
\end{lemma}

\begin{proof}
We proceed by induction on $t$.
The greedy farthest-first selection rule implies that, for the already selected points,
the following inequality holds: $\min_{\substack{i,j \le t; i\neq j}} \|\vecX_i - \vecX_j\|_2 \ge r_t^e$.
As a consequence, the Euclidean balls $B(\vecX_i, r_t^e/2)$ centered at $\vecX_i$ with radius $r_t^e/2$ do not intersect for $i = 1, \ldots, t$.

Since the entire set $\mathcal{D}$ is contained in a ball of radius $\Delta$
there exists a center $\mathbf{c} \in \mathbb{R}^d$ such that for any $i = 1, \dots, t$:
\begin{equation}
B(\vecX_i, r_t^e/2)
\subseteq
B \left(\mathbf{c}, \Delta + r_t^e/2 \right).
\end{equation}

Let
$V_p(r) := \frac{\pi^{p/2} r^d}{\Gamma\!\left(1+\tfrac{p}{2}\right)}$
denote the volume of a $p$-dimensional Euclidean ball of radius $r$.
By comparing volumes of the disjoint balls and the enclosing ball, we obtain
\begin{equation}
t\,V_p\bigl(r_t^e/2\bigr)
\le
V_p\bigl(\Delta + r_t^e/2\bigr).
\end{equation}

Introducing the normalized variable $y = \frac{r_t^e}{2\Delta} \in (0,1],$
the above inequality can be rewritten as $t\,y^p \le (1+y)^p.$
Taking the $p$-th root of both sides, we obtain $y \le \frac{1}{t^{1/p}-1}.$
Substituting back the definition of $y$ completes the proof.
\end{proof}

In order to apply this result to the kernel-induced dissimilarity used in our setting, we require the following bi-Lipschitz-type condition to hold on the entire set $\mathcal{D}$: there exist constants $C_1, C_2$, such that for any
$\vecX, \vecY \in \mathcal{D}$ it holds that
$C_1 \lVert \vecX - \vecY \rVert_2
\le d_k(\vecX, \vecY) := 2\bigl(1 - k(\vecX, \vecY)\bigr)
\le C_2 \lVert \vecX - \vecY \rVert_2$.

For a finite set $\mathcal{D}$, this condition is satisfied by choosing
\begin{equation}
C_1 =
\min_{\substack{\vecX, \vecY \in \mathcal{D}\\ \vecX \neq \vecY}}
\frac{d_k(x,y)}{\|x - y\|_2},
\quad
C_2 =
\max_{\substack{\vecX, \vecY \in \mathcal{D}\\ \vecX \neq \vecY}}
\frac{d_k(\vecX, \vecY)}{\|\vecX - \vecY\|_2},
\end{equation}
where $C_2 = C < \infty$ since the maximum is taken over a finite set, and
$C_1 > 0$ provided that $\|\vecX - \vecY\|_2 \neq 0$ for all $\vecX \neq \vecY$ (i.e., $\mathcal{D}$ contains no duplicate points).
Thus, given Lemma~\ref{lem:covering-radius-dim}, the kernel-induced covering radius satisfies:
\begin{equation}
r_t \le C\frac{2\,\Delta}{t^{1/p} - 1}.
\end{equation}

\subsection{Proof of Theorem~\ref{th:imse}}

For any $\vecX \in \mathcal{D}$, on the high-probability event defined in
Lemma~\ref{lem:concentration}, the following inequality holds: $\bigl(f(\vecX) - \mu_{t-1}(\vecX)\bigr)^2 \le \beta_t \sigma_{t-1}^2(\vecX)$. 
Summing over all points in $\mathcal{D}$, where $|\mathcal{D}| = n$, we obtain
\begin{equation}
\label{eq:et_bound}
E_t \le \beta_t \frac{1}{n}\sum_{\vecX \in \mathcal{D}} \sigma_{t-1}^2(\vecX).
\end{equation}
Denoting $V_{t-1} := \sum_{x \in \mathcal{D}} \sigma_{t-1}^2(\vecX)$, we obtain
$E_t \le \beta_t V_{t-1}/n$.

Using Lemma~\ref{lem:upper-variance}, we obtain $V_{t-1} \le A_{t-1} + n \sigma^2,$ where $A_{t-1} := \sum_{\vecX \in \mathcal{D}} 2\bigl(1 - \max_{j\le t-1} k(\vecX, \vecX_j) \bigr)$. 
By Lemma~\ref{lem:covering-radius-dim}, the quantity $A_{t-1}$ admits the following upper bound:
\begin{equation}
A_{t-1}
\le
n C\frac{2\,\Delta}{(t-1)^{1/p}-1}.
\end{equation}
\noindent
Substituting this estimate into the bound for $E_t$ in~\eqref{eq:et_bound}, we obtain the theorem statement.

\subsection{Proof of Theorem~\ref{th:rank_error}}\label{app:rank_error}

By construction of the Gaussian process posterior, we have
\begin{equation}
\mathbb{E} \left[(f_i(\vecX) - \hat{f}_{i, t - 1}(\vecX))^2 \mid \mathcal{D}_{t-1}\right]
= \sigma^2_{t-1}(\vecX), 
\quad i = 1, 2.
\end{equation}
Define the approximation errors $\varepsilon_i(\vecX) = f_i(\vecX) - \hat{f}_{i, t - 1}(\vecX), \, i = 1, 2$. Since the two Gaussian process models are independent, the random variables $\varepsilon_1(x)$ and $\varepsilon_2(x)$ are independent, and therefore
\[
e(\vecX)|\mathcal{S}_{t-1} = \hat{g}_{t-1}(x) - g(x) = \varepsilon_2(x) - \varepsilon_1(x)
\sim \mathcal{N} \left(0, 2 \sigma^2_{t-1}(\vecX) \right).
\]

We now reformulate the error event in a more convenient form.
Conditionally on $\mathcal{S}_{t-1}$, the quantity $\hat{g}_{t - 1}(\vecX)$ is fixed and
$g(\vecX) = \hat{g}_{t-1}(\vecX) - e(\vecX)$.
If $\hat{g}_{t - 1}(\vecX) > 0$, a ranking error occurs when $g(\vecX) < 0$, i.e.,
\begin{equation}
\hat{g}_{t-1}(\vecX) - e(\vecX) < 0
\quad\Longleftrightarrow\quad
e(\vecX) > \hat{g}_{t-1}(\vecX).
\end{equation}
An analogous argument holds for $\hat{g}_{t-1}(\vecX) < 0$.
Hence, a ranking error occurs if and only if $|e(x)| > |\hat{g}_{t-1}(\vecX)|.$

Let $F(\cdot)$ denote the cumulative distribution function of the standard normal
distribution. Since $e(\vecX) \sim \mathcal{N}(0, 2 \sigma^2_{t-1}(\vecX))$, the probability of a ranking error at point $\vecX$ can be written as
\begin{equation}
\mathbb{P} \bigl(|e(\vecX)| > |\hat g_{t-1}(\vecX)|\bigr)
= 2F \left(
-\frac{|\hat g_{t-1}(\vecX)|}{\sqrt{2} \sigma_{t-1}(\vecX)}
\right).
\end{equation}
Given the standard Gaussian tail bound $F(-a) \le \tfrac12 \exp(-a^2/2)$,
\begin{equation}
\mathbb{P}\bigl(|e(\vecX)| > |\hat g_{t-1}(\vecX)|\bigr)
\le
\exp\!\left(
-\frac{|\hat g_{t-1}(\vecX)|^2}{4 \sigma^2_{t-1}(\vecX)}
\right).
\end{equation}

Taking expectations over the dataset pool yields
\begin{equation}
\mathbb{E}[\mathrm{Err}_t|\mathcal{S}_{t-1}]
\le
\frac{1}{n}
\sum_{\vecX \in \mathcal{D}}
\exp \left(
-\frac{|\hat g_{t-1}(\vecX)|^2}{4 \sigma^2_{t - 1}(\vecX)}
\right).
\end{equation}

From the uniform upper bound on the posterior variance derived previously, we have
$\sigma^2_{t-1}(\vecX) \le S_t$ for all $\vecX \in \mathcal{D}$, where
\begin{equation}
S_t :=
C \frac{2\,\Delta}{(t - 1)^{1/p}-1}
+ \sigma^2.
\end{equation}
Let $\hat g_{\min}:=\min_{\vecX \in \mathcal{D}}|\hat g_{t-1}(\vecX)|$.
Then
\begin{equation}
\frac{1}{n}
\sum_{\vecX \in \mathcal{D}}
\exp \left( -\frac{|\hat g_{t-1}(\vecX)|^2}{4 \sigma^2_{t-1}(\vecX)} \right)
\le
\exp \left( -\frac{|\hat g_{\min}|^2}{4 S_t} \right).
\end{equation}
Consequently, the expected fraction of ranking errors satisfies
\begin{equation}
\mathbb{E}[\mathrm{Err}_t|\mathcal{S}_{t-1}]
\le
\exp \left( -\frac{|\hat g_{\min}|^2}{4 S_t} \right).
\end{equation}

\subsection{Cosine specialization}
\label{subsec:cosine-geometry}

For cosine similarity, define normalized vectors
$u_x=x/\|x\|_2$. Then
$d_k(x,y)=2(1-\cos(x,y))=\|u_x-u_y\|_2^2$.
Thus farthest-first with $d_k$ is equivalent to Euclidean
farthest-first on the normalized set $\{u_x:x\in D\}$, since
the square function is monotone and therefore does not change
the maximizer in the greedy step. Consequently, the Euclidean
covering bound applies to the normalized points:
if $\Delta_u=\operatorname{diam}(\{u_x:x\in D\})$, then
$r_t^{(e)}\le 2\Delta_u/(t^{1/p}-1)$ and the corresponding
$d_k$-covering radius satisfies
$r_t=(r_t^{(e)})^2\le (2\Delta_u/(t^{1/p}-1))^2$.
Since $\Delta_u\le 2$ on the unit sphere, constants can be
absorbed into the covering constant used in Theorem~\ref{th:imse}.

\section{Benchmark Scale Analysis}

The main TSC--RecSys comparison may confound domain effects with benchmark scale: TSC has 112 datasets and 35 models, whereas RecSys has only 30 datasets and 9 models. To isolate this factor, we construct RecSys-sized TSC mini-benchmarks. For each of $R=100$ repetitions, we sample $|\mathcal D_r|=30$ datasets and $|\mathcal M_r|=9$ models uniformly without replacement from the full TSC benchmark, and define the reference ranking using all scores on $\mathcal D_r \times \mathcal M_r$. We then run the standard dataset-subset protocol on each mini-benchmark for $k=2,\ldots,20$, with $\alpha=0.8$ and $T=50$ training splits: each split samples $\lfloor \alpha |\mathcal D_r| \rfloor$ candidate datasets, selects $k$ of them, and compares the induced ranking to the mini-benchmark reference ranking using MAE, Spearman $\rho$, NDCG@5, and Kendall $\tau$. Metric curves are aggregated by trapezoidal AUC. We compare Random, K-Means, and cosine/Euclidean farthest-first, fixing each method to its best representation from Table~\ref{tab:auc_combined}. To compare across scales, we report relative AUC gain over Random: $(\mathrm{AUC}_{\rm rand}-\mathrm{AUC}_{m})/\mathrm{AUC}_{\rm rand}$ for MAE and $(\mathrm{AUC}_{m}-\mathrm{AUC}_{\rm rand})/\mathrm{AUC}_{\rm rand}$ for higher-is-better metrics, so positive values always indicate improvement.
The resulting relative gains are reported in Table~\ref{tab:scale_analysis}; positive values indicate that the corresponding method outperforms Random under the same benchmark scale.

\begin{table}[t]
\centering
\caption{Relative AUC gain over Random (\%, higher is better). Each cell reports $\Delta$MAE; $\Delta\rho/\Delta$NDCG@5/$\Delta\tau$.}
\label{tab:scale_analysis}
\scriptsize
\setlength{\tabcolsep}{2.5pt}
\begin{tabular}{lccc}
\toprule
Method & TSC-full & TSC $30{\times}9$ & RecSys \\
\midrule
Cosine FAFI
& $+12.8;\ +3.4/+0.5/+6.1$
& $+3.6;\ -9.7/-11.5/-6.1$
& $+6.3;\ -0.3/-0.5/+0.5$ \\
K-Means
& $+14.1;\ +1.3/+0.2/+2.9$
& $+5.4;\ +11.7/+15.2/+5.7$
& $+8.9;\ -1.8/+0.1/-1.8$ \\
Euclidean FAFI
& $+17.3;\ +1.8/+0.7/+3.7$
& $-1.9;\ +9.5/+5.8/+6.0$
& $+4.5;\ -3.2/-0.3/-2.9$ \\
\bottomrule
\end{tabular}
\end{table}

At full TSC scale, all non-random methods improve over Random on both MAE and rank-based metrics. After downscaling TSC to $30{\times}9$, selection becomes more method-dependent: K-Means remains robust and improves all metrics, Euclidean farthest-first preserves rankings but slightly worsens MAE, while cosine farthest-first becomes unstable. RecSys differs from downscaled TSC: methods improve MAE but do not consistently improve rank-based metrics. Thus, benchmark scale partly affects stability, but scale alone does not explain the RecSys behavior; representation quality and domain-specific rank geometry remain crucial.

\end{document}